\documentclass[journal]{IEEEtran}

\usepackage{xcolor,soul,framed} %,caption

\usepackage[pdftex]{graphicx}

\usepackage[cmex10]{amsmath}
\usepackage{array}

\usepackage[utf8]{inputenc} % allow utf-8 input
\usepackage[T1]{fontenc}    % use 8-bit T1 fonts
\usepackage{hyperref}
\usepackage{url}            % simple URL typesetting
\usepackage{booktabs}       % professional-quality tables
\usepackage{amsfonts}       % blackboard math symbols
\usepackage{nicefrac}       % compact symbols for 1/2, etc.
\usepackage{cite}
\usepackage{microtype}      % micro typography
\usepackage{amsmath}
\usepackage{color}          % red blue green
\usepackage{theorem}
\usepackage{amssymb}
\usepackage{graphicx}
\usepackage{subfigure}
\usepackage{multirow}

\usepackage{algorithmicx}
\usepackage{algpseudocode}
\usepackage{algorithm,algpseudocode,float}
\newtheorem{theorem}{Theorem}
\newtheorem{lemma}{Lemma}

\begin{document}
% \bstctlcite{IEEEexample:BSTcontrol}
    \title{PAC-Bayes Bounds for Meta-learning with \\ Data-Dependent Prior}
  \author{Tianyu Liu,~\IEEEmembership{Student Member,~IEEE,}
      Jie~Lu,~\IEEEmembership{Fellow,~IEEE,}
      Zheng~Yan,~\IEEEmembership{Member,~IEEE,}
      and~Guangquan Zhang,~\IEEEmembership{Member,~IEEE}% <-this % stops a space
}

% \markboth{IEEE 
% }{Roberg \MakeLowercase{\textit{et al.}}: test}

% ====================================================================
\maketitle

% === ABSTRACT ====================================================================
% =================================================================================
\begin{abstract}
%\boldmath
By leveraging experience from previous tasks, meta-learning algorithms can achieve effective fast adaptation ability when encountering new tasks. However it is unclear how the generalization property applies to new tasks. Probably approximately correct (PAC) Bayes bound theory provides a theoretical framework to analyze the generalization performance for meta-learning. We derive three novel generalisation error bounds for meta-learning based on PAC-Bayes relative entropy bound. Furthermore, using the empirical risk minimization (ERM) method, a PAC-Bayes bound for meta-learning with data-dependent prior is developed. Experiments illustrate that the proposed three PAC-Bayes bounds for meta-learning guarantee a competitive generalization performance guarantee, and the extended PAC-Bayes bound with data-dependent prior can achieve rapid convergence ability.
\end{abstract}

% === KEYWORDS ====================================================================
% =================================================================================

% \fi
\IEEEpeerreviewmaketitle

% ====================================================================
% ====================================================================
% ====================================================================

% === I. INTRODUCTION =============================================================
% =================================================================================
\section{Introduction}\label{section_1}

\IEEEPARstart{M}{achine} learning models often require training with a large number of samples, for example, image classification issue \cite{xing2017deep,  tang2018deep,luo2013sparse,liu2020survey}. Besides, traditional machine learning algorithms mainly focus on a single task. But it is generally difficult to collected so much labelled data. So how to train a model when only a small amount of data is available? More to the point, since humans can learn new skills much faster and more effectively, how can we build such a model, which can reflect aspects of human learning? That is what meta-learning sets out. \textit{Meta-learning} -- or ``learning to learn \cite{schmidhuber1987evolutionary}” -- is capable of accurately adapting or generalizing to new tasks and new environments that not encountered during training time. Using the experience acquired on previous tasks, meta-learning can adapt to new tasks quickly, even in the face of scant data. The meta-learning algorithms can be divided into three major categories \cite{finn2019meta}: \textit{black-box algorithms}, \textit{non-parametric methods} and \textit{optimization-based algorithms}.

% {\bf Black-box adaptation.}
The core idea of black-box adaptation meta-learning is to train a neural network to represent a meta-learner. With the aim of achieving a fast adaptation ability, \cite{santoro2016meta} a meta-learning algorithm is presented with memory-augmented neural networks, which can summarize and storage important knowledge. When facing new learning tasks, the memory-based method can extract certain skills it has experienced to assist in the current process. In order to adapt to access past experiences, a simple neural attentive-learner (SNAIL) is proposed by \cite{mishra2017simple}. By using attention architectures established in the meta-learner, SNAIL can determine what pieces of information it needs to select from its experience it gathers. SNAIL architectures are easier to train than traditional RNN, such as LSTM.

% {\bf Non-parametric methods.}
Non-parametric methods try to utilize a non-parametric learner as meta-learner instead of parametric models. Non-parametric methods are simple and perform well in few-shot learning. \cite{bromley1994signature} proposes a siamese neural network, which contains two sub-networks with same weights. During training phase, the two sub-networks can extract features from two different input vectors, and then compute the distance between the two feature vectors. Matching networks are another non-parametric method, which is presented by \cite{vinyals2016matching}. In order to learn from a few examples, matching network framework learns a net structure that maps few labelled training datasets and an unlabelled instance to its label. Combined with recent advances in attention and memory, the matching networks enable rapid learning. Besides, \cite{snell2017prototypical} proposes a prototypical network, where classification problem is regarded as finding the prototype center of each category in the semantic space and then predict the category of the new sample by the nearest neighbor classifier. This method mainly combines the prototype network with clustering algorithm.

% {\bf Optimization-based inference.}
Different from two aforementioned algorithms, optimization-based meta-learning algorithms learn to train the parameter vector to represent the meta-learner through optimization. In the traditional gradient-descent approach, optimization updates rules, for example the learning step, are still hard to design. \cite{andrychowicz2016learning} considers this issue as a learning problem, allowing the optimization algorithms learn to exploit update rules structure in an automatic way. Furthermore, \cite{ravi2016optimization} propose another LSTM-based meta-learner model by combining gradient descent and LSTM algorithm, which is applied to train neural network.
% In that method, the learning size and other parameters of the gradient descent method correspond to the parameters in the LSTM network.
% to learn the exact optimization algorithm
In order to extract common knowledge from previous task so as to achieve fast adaptation ability, \cite{finn2017model} proposes a model agnostic meta-learning (MAML) algorithm. The key idea of MAML is to learn a set of initialization parameter that allows efficient learning of new tasks. However, MAML requires the computation of second-order derivative which may exhibit instabilities. Therefore \cite{nichol2018reptile} present a scalable meta-learning algorithm, called Reptile, which does not calculate any second derivatives. Besides, \cite{antoniou2018train} addresses the training of MAML and propose several tricks to improve the stability of MAML.

One of a majority challenges in few-shot learning is task ambiguity. 
\cite{finn2018probabilistic} proposes a probabilistic MAML, which tries to incorporate a parameter distribution with neural network that is trained via a variational lower bound.  In order to improve the robustness of MAML, \cite{kim2018bayesian} propose a Bayesian MAML algorithm. Compared with a point estimate or a simple Gaussian approximation in fast adaptation phase, this algorithm is capable of learning very complex uncertainty structure. The Bayesian MAML outperforms vanilla MAML in terms of accuracy and robustness. Furthermore, based on Bayesian inference framework and variational inference, \cite{lee2019learning} propose a new Bayesian task-adaptive meta-learning (Bayesian TAML) algorithm for imbalanced and out-of-distribution tasks. In addition, several improved MAML are also introduced, such as Alpha MAML \cite{behl2019alpha}, meta-learning with latent embedding optimization \cite{rusu2018meta} and Bayesian hierarchical modeling based MAML \cite{grant2018recasting}.

Although meta-learning algorithms provide a powerful inductive biases based on various tasks, even with those which comprise only limited data, its generalization performance is poorly understood. PAC-Bayes theory, known as generalization error bounds theory, provides a theoretical framework for estimating the generalization performance of the machine learning model.

The first PAC-Bayes theory was established by McAllester \cite{mcallester1999some}, which provides generalization error upper bounds for the performance of randomized learning algorithms. Then this method was subsequently used to analyze the generalization-error bound of the stochastic neural network \cite{langford2001not}. PAC-Bayes bound theories were meant for a wide range of approximate Bayesian GP classification issues \cite{seeger2002pac}, \cite{catoni2007pac}. One source \cite{neyshabur2017exploring} tries to explain the generalization in neural network from the view of norm-based control, sharpness and robustness, and attempts to build a connection between sharpness and PAC-Bayes theory. The systemically undertaken study is addressed with a view to training stochastic neural networks based on the PAC-Bayes bounds in \cite{perez2020tighter}. 

That PAC-Bayes theory is only suitable for bounded loss function and i.i.d data. PAC-Bayesian bounds tailored for the sub-Gaussian or sub-Gamma loss family, such as negative log-likelihood function, is also developed by \cite{germain2016pac} and \cite{alquier2016properties}. However, those algorithms require a distribution parameter, such as a variance factor and a scale parameter. Therefore \cite{holland2019pac} proposes an exponential bound under the assumption that the first three moments of the loss distribution are bounded. By introducing the special boundedness condition,  \cite{haddouche2020pac} expands the PAC-Bayesian theory to learning problems with unbounded loss functions.
 
Recently, there has been a gradually increasing interest in research on overparameterized deep neural networks and SGD.  \cite{london2017pac} study the generalization of randomized learning algorithms. trained with SGD. Inspired by \cite{langford2001not}, \cite{dziugaite2017computing} obtains nonvacuous generalization numerical bounds for deep stochastic neural network classifiers with many more parameters than are present in the training data. The first non-vacuous generalization bound for compressed networks applied to the ImageNet classification problem is provide in \cite{zhou2018non}. Moreover, \cite{dziugaite2018entropy} further investigates the relationship between generalization performance and SGD. 

As mentioned above, PAC-Bayesian bound is only valid for stochastic classifiers, although a growing body of literature illustrates efforts to construct PAC-Bayes bounds on deterministic classifiers. To fill this gap, \cite{miyaguchi2019pac} develops a PAC-Bayesian transportation bound, by unifying the PAC-Bayesian analysis and the chaining method. This generalization error bound relates the distance between any two predictors, both for stochastic classifiers and deterministic classifiers. A new perturbation bounds for feedforward neural networks is derived based on the sharpness of a model class by \cite{neyshabur2017pac}. In addition, \cite{nagarajan2019deterministic} presents a general PAC-Bayesian framework for the deterministic and uncompressed neural network by leveraging the noise-resilience of deep neural networks on training data.

In order to achieve tighter generalization error bounds, \cite{parrado2012pac} proposes two alternative prior distributions: one is to learn a prior distribution from a separate training data set which is not used in computing the bound, and another is to consider an expectation prior. \cite{lever2013tighter} further investigates that PAC-Bayes bound with localized prior distribution defined in terms of the data generating distribution. Under the stability of the hypothesis, a Gaussian prior distribution, informed by the data-generating distribution and centered at the expected output, is proposed for the SVM classifier \cite{rivasplata2018pac}. More discussion can be seen in \cite{lever2010distribution} and \cite{oneto2016pac}. Furthermore, because data distribution is usually unknown, \cite{dziugaite2018data} develops a PAC-Bayes bound via $\epsilon$-differentially private data-dependent prior. 

PAC-Bayes theory provides a theoretical framework for the generalization performance analysis of meta-learning. This theory can be considered as a generalized framework which is more resistant to over-fitting and that yields a generalization error upper bound that holds with an arbitrarily high probability. For meta learning, \cite{pentina2014pac} provides a generalization error bound within the PAC-Bayes framework for lifelong learning. Furthermore, two principled algorithms are implemented, including parameter and representation transfers. More recently,  \cite{amit2018meta} develops a theoretical framework for meta-learning, allowing extended various PAC-Bayes bounds to meta-learning. To add to this, \cite{Huang_Huang_Li_Li_2020} considers the scenario in which a common model set is used for model averaging via a model selection procedure that accounts for the model’s uncertainty. Two data-based algorithms are proposed to obtain ideal priors for model averaging.

Specifically, a gradient-based algorithm which minimizes an objective function derived from PAC-Bayes bounds is also applied to training deep neural networks. The tighter bounds might achieve a better generalization performance. Besides, in PAC-Bayes theory, prior distribution is selected randomly before learning. Generally, with PAC-Bayes, the generalization error upper bound is primarily determined by the distance between prior and posterior distributions. Obviously the choice of prior distribution affects the performance of the PAC-Bayes bound significantly. 

Motivated by the previous discussions, three novel generalization error bounds for meta-learning are presented. Furthermore, a data-based approach for adjusting prior distribution is developed, and the specific implementations of those algorithms are given. The main contributions are concluded as follows

% based on the PAC-Bayes relative entropy theory, we propose three novel PAC-Bayes bounds for meta-learning, including {\it meta-learning PAC-Bayes $\lambda$ bound}, {\it meta-learning PAC-Bayes quadratic bound}, and . First of all, we are going to investigate the {\it PAC-Bayes relative entropy bound} and then extend those bounds to meta-learning.

\begin{itemize}
  \item In order to improve generalization performance, based on the PAC-Bayes relative entropy theory, {\it meta-learning PAC-Bayes $\lambda$ bound} and {\it meta-learning PAC-Bayes quadratic bound} are proposed;
  \item Using the variational Kullback-Leibler (KL) bound, {\it meta-learning PAC-Bayes variational bound} is investigated, which can achieve tighter generalisation error bound by taking the piecewise combination of the two above-mentioned meta-learning bounds;
  \item Based on ERM method, a PAC-Bayes bound for meta-learning with data-dependent prior is developed by adjusting the priors to attain fast convergence ability;
  \item Empirical demonstration illustrates that the proposed algorithms achieve competitive generalization guarantee and better convergence performance.
\end{itemize}

The rest of this paper is organized as follows. The classical PAC-Bayes bounds for both single task and meta-learning is introduced in Section \ref{section_3_Preliminaries}. Section \ref{section_4_bound} investigates three novel PAC-Bayes bounds for meta-learning, based on the PAC-Bayes relative entropy bound theory. A PAC-Bayes bound for meta-learning with data-dependent prior is developed in Section \ref{section_5_bound_dp}. The implementation details are described in Section \ref{section_6_practical}.  Section \ref{section_7_experiments} provides numerical examples to verify the proposed algorithms. Finally, Section \ref{section_8_conclusion} draws some conclusions.

\section{Preliminaries: PAC-Bayes theorem}\label{section_3_Preliminaries}

In the classical statistical learning model setting, a set of dependent samples $ S = \{z_i\}_{i=1}^m $ is randomly drawn from the unknown data distribution $ \mathcal{D}  $ over space $\mathcal{Z}$. In the supervised learning, $\mathcal{Z} = \mathcal{X} \times \mathcal{Y}$, where $\mathcal{X} \subset \mathbb{R}^d$ and $\mathcal{Y} \subset \mathbb{R}$, each sample $Z_i = (X_i, Y_i)$ consists of an input $X_i$ and its corresponding label $Y_i$. The learning objective is to find a classifier $h \in \mathcal{H}$ that predicts the label and minimizes the expected loss $\mathbb{E}\{\ell(h,z)\}$, where $\mathcal{H}$ is considered as the hypothesis space and $\ell(h,z): \mathcal{H} \times \mathcal{Z} \to \mathbb{R}$ is the loss function which are used to measure the performance of prediction. For the classification problems, the loss function is always bounded in $[0,1]$. In the statistical inference stage, the core idea of machine learning is to minimize the {\it expected error} $er(h, \mathcal{D})$ under the data distribution $\mathcal{D}$ 
\begin{eqnarray}
\begin{array}{cc}
er(h, \mathcal{D}) = \mathbb{E}_{z\sim\mathcal{D}}\ell(h,z).
\end{array}
\end{eqnarray}
Since the data distribution $\mathcal{D}$ is unknown,  {\it generalization error} $er(h, \mathcal{D})$ cannot be calculated. Therefore, the {\it empirical error} $\widehat{er}(h,S)$ gives an observable estimation
\begin{eqnarray}
\begin{array}{cc}
\widehat{er}(h,S) = \frac{1}{n} \sum^{n}_{i=1}\ell(h,z_i).
\end{array}
\end{eqnarray}
Under certain neural network conditions,  in order to minimize the empirical risk, a single classifier $h_{w} \in \mathcal{H}$ is selected. However, this may cause that the learned classifier $ \widehat{h}_{w} $ to form too close a fit to a limited set of data points $ S $ --- creating a case of over-fitting, which can be measured by $ er( \widehat{h}, \mathcal{D}) - \widehat{er}( \widehat{h}, S) $. Various methods can be used to avoid over-fitting, various methods are used, including complexity regularization. 

\subsection{PAC-Bayes bounds for single task}
PAC-Bayes theory, known as generalization error bound theory, is a framework for theoretical generalization performance analysis of a machine-learning model.

In the PAC-Bayes theory, in contrast with the classical neural network which aims to learn data-dependent parameter weights, the probability neural network is employed to learn a data-dependent distribution over weights. Specifically, based on the ``prior" distribution $ P \in \mathcal{M} $, PAC-Bayes bound tries to learn a posterior distribution $Q(S,P) \in \mathcal{M}$ from training data $S$, where $ \mathcal{M} $ denotes the set of distributions over hypothesis space $\mathcal{H}$. Then {\it generalization error} $er(Q, S)$ and {\it empirical error} $\widehat{er}(Q, S)$ are defined as the expectation over posterior distribution $Q$, such as $er(Q, \mathcal{D}) \triangleq \underset{h \sim Q}{\mathbb{E}} er(h, S)$ and $\widehat{er}(Q, \mathcal{D}) \triangleq \underset{h \sim Q}{\mathbb{E}} er(h, S)$. The first PAC-Bayes generalization theory for single task learning issue has been proposed by \cite{mcallester1999some}.

\begin{lemma} \label{mcallester1999}
(McAllester's single-task bound \cite{mcallester1999some}). Let $P \in \mathcal{M}$ be some prior distribution over $\mathcal{H}$. Then for any $\delta \in (0,1]$, the following inequality holds uniformly for all posteriors distributions $Q \in \mathcal{M}$ with probability at least $1 - \delta$
\begin{eqnarray}
\begin{array}{cc}
er(Q, \mathcal{D}) \leq \widehat{er}(Q, S) + \sqrt{ \frac{ D(Q\|P) + \log \frac{m}{\delta}} {2(m-1)}}.
\end{array}
\end{eqnarray}
\end{lemma} 

Here $D\left(\rho \| \rho_{0}\right)$ is the KL divergence, which measures the difference between two distributions:
\begin{eqnarray}
\begin{array}{cc}
\mathrm{KL}\left(\rho \| \rho_{0}\right) \stackrel{\mathrm{def}}{=} \underset{c \sim \rho}{\mathbb{E}}\left[\ln \frac{ {\rm d} \rho(c)}{ {\rm d} \rho_{0}(c)}\right],
\end{array}
\end{eqnarray}
where $\frac{ {\rm d} \rho(c)}{ {\rm d} \rho_{0}(c)}$ is the Radon-Nikodym derivative of $\rho$ with respect to $\rho_{0} .$ The Radon-Nikodym derivative can be substituted with the ratio of Probability Density Functions (PDF) if they exist.

Generally, a PAC-Bayes generalization theory attempts to balance the discrepancy between a prior distribution $P$ with posterior distribution $Q$, and empirical risk $\widehat{er}(Q, S)$. Here, we should emphasize that prior distribution $P$ is selected randomly before learning, which must not be dependent on training data. Besides, posterior distribution $Q(S,P)$ does not necessarily have to be the traditional Bayesian posterior distribution. The prior distribution $P$ is used chiefly to measure the distance of hypothesis space $\mathcal{H}$. Obviously, the choice of prior distribution $P$ significantly affects the performance of PAC-Bayes bound significantly.

\subsection{PAC-Bayes bounds for meta-learning}
In this subsection, the PAC-Bayes bound for meta-learning is introduced. meta-learning comprises two parts: meta-learner extracts common knowledge (prior knowledge) from different observed tasks, and base learner aims to adapt new tasks. In the PAC-Bayes meta-learning framework, we assume that all tasks belong to the same distribution $\mathcal{T}$. Different tasks share the same sample space $\mathcal{Z}$ and loss function $\ell(h,z): \mathcal{H} \times \mathcal{Z} \to \mathbb{R}$. For each observed task $\tau_i$, the corresponding $S_i$ are generated from the unknown distribution $S_i \sim \mathcal{D}^{m_i}_i$, where $m_i$ is the number of training samples for task $i$. As mentioned before, the {\it meta-learner} tries to extract common knowledge $P \sim \mathcal{M(H)}$ from tasks $\tau$ before, based on the prior information $P$ and new task's data $S$, {\it base learner} learns the posterior information $Q(S,P): \mathcal{Z}^m \times \mathcal{M(H)} \to \mathcal{M(H)}$ to infer the process applied when faced with new tasks. Here prior information $P$ and posterior information $Q(S,P)$ are represented as distribution of neural network weights $w$, characterized as the mean and co-variance.

For the meta-learning generalization error bound theory, based on the {\it hyper prior} $\mathcal{P}$, which is a distribution over prior distribution $P$, {\it meta-learner} aims to learn a {\it hyper posterior} $\mathcal{Q}(P)$ by utilizing the observed tasks. When encountering new tasks, {\it base learner} samples a prior distribution $P$ from the hyper-posterior $\mathcal{Q}(P)$. Capitalizing on  observed samples of the new task, {\it base learner} infers a posterior distribution $Q(S,P)$. The performance of hyper-posterior $\mathcal{Q}$ can be measured by the expectation loss of prior $P$ when learning new tasks, the so-called {\it generalization error}
\begin{eqnarray}
\begin{array}{cc}
er(\mathcal{Q}, \tau) \triangleq \underset{P \sim \mathcal{Q}}{\mathbb{E}}\ \underset{(\mathcal{D}, m) \sim \mathcal{T}}{\mathbb{E}}\ \underset{S \sim \mathcal{D}^{m}}{\mathbb{E}}\ \underset{h \sim Q(S, P)}{\mathbb{E}}\ \underset{z \sim \mathcal{D}}{\mathbb{E}}\ \ell(h, z).
\end{array}
\end{eqnarray}

While $er(\mathcal{Q}, \tau)$ is not commutable in practice, nevertheless, we can estimate this generalization error by the {\it empiric error}
\begin{eqnarray}
\begin{array}{cc}
\widehat{er}\left(\mathcal{Q}, S_{1}, \ldots, S_{n}\right) \triangleq \underset{P \sim \mathcal{Q}}{\mathbb{E}} \frac{1}{n} \sum_{i=1}^{n} \widehat{e r}\left(Q\left(S_{i}, P\right), S_{i}\right).
\end{array}
\end{eqnarray}

Different to the single-task PAC-Bayes bound, in meta-learning, meta-learner chooses a hyper-prior distribution $\mathcal{P}$ over prior distribution $P$, following observed samples for all training tasks, and updates it to hyper-posterior distribution $\mathcal{Q}$. Base learner selects a prior distribution $P \sim \mathcal{P}$ and updates it to posterior distribution $Q(S,P)$ when learning new tasks.

Firstly, we introduce the classical extended PAC-Bayes bound proposed by \cite{amit2018meta} as follows
\begin{lemma} \label{PAC_Bayes_meta_classic}
(Classical meta-learning PAC-Bayes bound \cite{amit2018meta}). Let $\mathcal{P} \in \mathcal{M}$ be some hyper-prior distribution over $\mathcal{H}$, and $Q$ be a base learner. Then for any $\delta \in (0,1]$, the following inequality holds uniformly for all hyper-posteriors distributions $\mathcal{Q} \in \mathcal{M}$ with probability at least $1-\delta$
\begin{eqnarray}\label{Classical_extend_PAC}
\begin{array}{l}
er( \mathcal{Q}, \tau ) \leq \frac{1}{n} \sum_{i=1}^{n} \underset{P \sim \mathcal{Q}}{\mathbb{E}} \hat{er} \left( Q_{i} \left( S_i, P\right), S_i \right) \\
+\frac{1}{n} {\sum_{i=1}^{n}} \sqrt{ \frac{ \big( D ( \mathcal{Q} || \mathcal{P} ) + \underset{ P \sim \mathcal{Q} }{ \mathbb{E} } D \left( Q  || P \right) + \log \frac{ 2 n m_i }{ \delta} \big) }{2\left( m_i - 1 \right)} } \vspace{2pt}\\
+ \sqrt{ \frac{1}{2(n-1)} \left( D ( \mathcal{Q} || \mathcal{P} ) + \log \frac{2n}{\delta} \right) }.
\end{array}
\end{eqnarray}
\end{lemma} 

It is obviously this PAC-Bayes meta-learning bound consists of empirical multi-task error plus two regularization terms. The first task-complexity term is the average of task-complexity terms of observed tasks, created by the finite number of samples in each observed tasks. This term converges to zero in the face of a large number of samples in each task. The second is an environment-complexity term, which is caused by a finite number of observed training tasks. Obviously, this term converges to zero if an infinite number of tasks is observed from the task environment.

\section{PAC-Bayes meta-learning bounds}\label{section_4_bound}
In this section, based on the PAC-Bayes relative entropy theory, we propose three novel PAC-Bayes bounds for meta-learning, including {\it meta-learning PAC-Bayes $\lambda$ bound}, {\it meta-learning PAC-Bayes quadratic bound}, and {\it meta-learning PAC-Bayes variational bound}. We begin by investigating the PAC-Bayes relative entropy bound and then extend those bounds to meta-learning.

\subsection{PAC-Bayes relative entropy bound}
With high probability $1-\delta$, PAC-Bayes bound theory provides a generalization performance guarantee for the learned model. The generalization error upper bound depends on empiric loss and a regularization item involves the distance between prior distribution and posterior distribution.

First of all, the PAC-Bayes relative entropy bound and its corresponding variants based on different inequalities are introduced. Similarly, in Theorem. \ref{mcallester1999}, with high probability $1-\delta$, the PAC-Bayes relative entropy bound holds that
\begin{eqnarray}
\begin{array}{rl}
\rm{kl} ( er(Q, \mathcal{D}) \| \widehat{er}(Q, S) ) \leq \frac{\mathrm{D} (Q \| P) + \log (\frac{2 \sqrt{n}}{\delta})}{n}.
\end{array}
\end{eqnarray}

% \operatorname{er} (Q, \mathcal{D}) \leq \widehat{er}(Q, S)

Here, as shown in (\ref{kl_Bernoulli}), $ \mathrm{kl} $ is known as the binary KL divergence, which is the divergence of two Bernoulli distributions with parameters $ q, q^{\prime} \in [0,1] $
\begin{eqnarray}\label{kl_Bernoulli}
\begin{array}{rl}
\mathrm{kl}\left(q \| q^{\prime}\right)=q \log \big( \frac{q}{q^{\prime}} \big) + (1-q) \log \big( \frac{1-q}{1-q^{\prime}} \big).
\end{array}
\end{eqnarray}

Obviously, with an arbitrarily high probability, the generalization error of the learned model is bounded by the summation of empirical loss, and a regularization element involves the distance between prior distribution and posterior distributions.

Applying the {\it refined Pinsker inequality} $ \mathrm{kl}(\hat{p} \| p) \geq \frac{(p-\hat{p})^{2}}{2 p} $, $ (\hat{p}, p \in(0,1), \hat{p}<p) $ yields
\begin{eqnarray}
\label{refined_pinsker_ineqn}
\begin{array}{l}
er(Q,\mathcal{D}) - \widehat{er}(Q, S) \leq \sqrt{2 er(Q,\mathcal{D}) \frac{\mathrm{D}\left(Q \| P \right) + \log \left(\frac{2 \sqrt{n} }{\delta} \right)}{n}}.
\end{array}
\end{eqnarray}

The PAC-Bayes relative entropy bound cannot be selected directly as the training objective function directly, because generalization error $er(Q, \mathcal{D})$ appears in the right side of the bound which cannot be used as an optimization objective. Therefore, the following two PAC-Bayes bounds are proposed.

% By solving this inequality, the following two PAC-Bayes bounds are proposed.

First, combining (\ref{refined_pinsker_ineqn}) with the inequality $\sqrt{ab} \leq \frac{1}{2}\left(\lambda a+\frac{b}{\lambda}\right), \lambda>0$ yields the {\it PAC-Bayes $\lambda$ bound}.
\begin{lemma} \label{PAC_Bayes_lambda}
(PAC-Bayes $\lambda$ bound \cite{thiemann2017strongly}). Let $P \in \mathcal{M}$ be some prior distribution over $\mathcal{H}$. Then for any $\delta \in (0,1]$ and $\lambda \in (0,2)$, the following inequality holds uniformly for all posteriors distributions $Q \in \mathcal{M}$ with probability at least $1-\delta$
\begin{eqnarray}
\begin{array}{cc}
er(Q,\mathcal{D}) \leq \frac{ \widehat{er}(Q, S) }{ 1 - \lambda / 2} + \frac{ \mathrm{KL} \left( Q \| P \right) + \log (2 \sqrt{n} / \delta) }{ n \lambda( 1 - \lambda / 2)}.
\end{array}
\end{eqnarray}
\end{lemma} 
Compared with classic generalization error upper bound (see Theorem \ref{mcallester1999}), with a reasonable selection of parameters $\lambda$, we may obtain a smaller upper bound.

Alternatively, one may view inequality (\ref{refined_pinsker_ineqn}) as a quadratic inequality on $\sqrt{er(Q,\mathcal{D})}$. Solving this inequality yields the following {\it PAC-Bayes quadratic bound}.

\begin{lemma} \label{PAC_Bayes_quadratic}
(PAC-Bayes quadratic bound \cite{perez2020tighter}). Let $P \in \mathcal{M}$ be some prior distribution over $\mathcal{H}$. Then for any $\delta \in (0,1]$, the following inequality holds uniformly for all posteriors distributions $Q \in \mathcal{M}$ with probability at least $ 1 - \delta $
\begin{eqnarray}
\begin{array}{cc}
er(Q, \mathcal{D}) \leq \big( \sqrt{ \widehat{er} \left(Q, S\right) + \varepsilon } + \sqrt{\varepsilon} \big)^2,
\end{array}
\end{eqnarray}
where $\varepsilon = \frac{ 1}{ 2m } \big( D\left(Q\|P\right) + \log\frac{2\sqrt{m}}{\delta} \big)$.
\end{lemma} 

It can be seen that when the generalization error is smaller (especially $er(Q, \mathcal{D})<1/4$), this bound is tighter (see \cite{perez2020tighter}) than the classical PAC-Bayes bound. When generalization error $er(Q,\mathcal{D})$ varies, different bounds can enact alternative different generalization performances. Motivated by this, \cite{dziugaite2020role} proposes a variational KL bound.

\begin{lemma}(Variational PAC-Bayes bound \cite{dziugaite2020role}). Let $P \in \mathcal{M}$ be some prior distribution over $\mathcal{H}$. Then for any $\delta \in (0,1]$, the following inequality holds uniformly for all posteriors distributions $Q \in \mathcal{M}$ with probability at least $ 1 - \delta $
\begin{equation}
er (h, \mathcal{D}) \leq {\rm min} \left\{
\begin{aligned}
& \widehat{er} \left(Q, S\right) + \varepsilon + \sqrt{ \varepsilon
( \varepsilon + 2 \widehat{er} \left(Q, S\right) ) } , \\
& \widehat{er} \left(Q, S\right) + \sqrt{ \varepsilon/2 },
\end{aligned}
\right.
\end{equation}
where $ \varepsilon = \frac{1}{m} \big( D\left(Q\|P\right) + \log\frac{2\sqrt{m}}{\delta} \big) $.
\end{lemma}
Contrasting with the two previously described PAC-Bayes bounds, PAC-Bayes variational bound can take a minimum of two bounds, which might achieve a tighter generalization error for upper bound.

% \frac{\varepsilon}{2}

\subsection{Extended PAC-Bayes kl bounds for meta-learning}
For meta-learning, careful design of the training tasks is required to prevent a subtle form of task overfitting, which means a learned meta-learner generalizes on training tasks but fails to adapt to new ones. This form of overfitting is known as the  {\it memorization} problem in meta-learning \cite{yin2019meta}. PAC-Bayes theory provides a theoretical framework for the analysis of the generalization performance in meta-learning. By selecting the PAC-Bayes bound as the training objective, one not only can reduce overfitting but also develop a deep neural network with a guaranteed generalization performance. Here, the tighter bounds can achieve enhanced results.

Motivated by this, and based on the PAC-Bayes relative entropy theory, in this section we propose three novel PAC-Bayes bounds for meta-learning. We begin with examining  {\it meta-learning PAC-Bayes $\lambda$ bound}, First of all, based on Lemma \ref{PAC_Bayes_lambda} and Lemma \ref{PAC_Bayes_quadratic}, {\it followed by investigations of meta-learning meta-learning PAC-Bayes quadratic bound} and {\it meta-learning PAC-Bayes variational bound}.

% The proof of those meta-learning PAC-Bayes bounds are shown in the Appendix. \ref{proof_meta_bound}.

% Based on the {\it PAC-Bayes relative entropy bound}, we propose two novel extended PAC-Bayes bounds, named as the extended PAC-Bayes $\lambda$ bound and extended PAC-Bayes quadratic bound. the extended PAC-Bayes $\lambda$ bound is shown in (\ref{Extend_PAC_lambda}) . 

\begin{theorem}\label{Extend_PAC_lambda}
(meta-learning PAC-Bayes $\lambda$ bound). Let $\mathcal{P}$ be some hyper-prior distribution, and $Q$ be the posterior distribution, which is also called as the base learner. Then for any $\delta \in(0,1]$ and $\lambda \in (0, 2)$, the following inequality holds uniformly for all hyper-posteriors distributions $\mathcal{Q} \in \mathcal{M}$ with probability at least $1-\delta$
\begin{equation}
\begin{array}{l}
er(\mathcal{Q}, \tau) \leq \frac{1}{n} \sum_{i=1}^{n} \frac{1}{ (1 - \lambda_0 /2)^2  } \underset{ P \sim \mathcal{Q} }{ \mathbb{E} } \hat{er} \left( Q, S_i \right)\\
+\frac{1}{n} \sum_{i=1}^{n}  \frac{ D( \mathcal{Q} || \mathcal{P} ) + \underset{P \sim \mathcal{Q} }{ \mathbb{E} } D\left( Q || P \right) + \log \frac{ 4n\sqrt{m_i} }{ \delta } }{ m_i\lambda(1-\lambda /2)^2 } \vspace{4pt}\\
+ \sqrt{ \frac{1}{2(n-1)} \left( D ( \mathcal{Q} || \mathcal{P} ) + \log \frac{2 n}{\delta} \right) }.
\end{array}
\end{equation}
\end{theorem} 
With the reasonable selection of $\lambda$, this meta-learning PAC-Bayes bound can attain a tighter generalization error bound. The proof of this meta-learning bound is shown as follows.

\begin{proof}
In this section, a briefly proof of the extended PAC-Bayes $\lambda$ bound is introduced.

% \begin{lemma}
% (PAC-Bayes $\lambda$ bound). Let $P \in \mathcal{M}$ be some prior distribution over $\mathcal{H}$. Then for any $\delta \in (0,1]$ and $\lambda \in (0,2)$, the following inequality holds uniformly for all posteriors distributions $Q \in \mathcal{M}$ with probability at least $1-\delta$
% \begin{eqnarray}
% \begin{array}{cc}
% er(Q,\mathcal{D}) \leq \frac{ \widehat{er}(Q, S) }{ 1 - \lambda / 2} + \frac{ \mathrm{KL} \left( Q \| P \right) + \log (2 \sqrt{n} / \delta) }{ n \lambda( 1 - \lambda / 2)}.
% \end{array}
% \end{eqnarray}
% \end{lemma} 

Let $n$ be the number of training tasks. The samples of task $i$ are $ z_{i, j}, j=1,...,K , K \triangleq m_{i},$ over the data distribution $ \mathcal{D}_{i}$. The bounded loss function is defined as $\ell(h, z)$. We define the prior distribution $P$ , which is sampled from hyper-prior distribution $\mathcal{P}$. The posterior distribution is defined as $Q = Q \left( S_{i}, P \right) $, which is sampled from hyper-posterior distribution $\mathcal{Q}$. Here as exemplified in \cite{amit2018meta}, `tuple hypothesis' is defined as $f=(P, h)$ where $ P \in \mathcal{M} $ and $ h \in \mathcal{H}$; `prior over hypothesis' is defined as $\pi \triangleq(\mathcal{P}, P)$, where $h$ is sampled from $P$. We note that  the 'posterior over hypothesis' can be any distribution (even sample dependent). In particular, the PAC-Bayes bound will hold for the following family of distributions over $ \mathcal{M} \times \mathcal{H}, \rho \triangleq \left( \mathcal{Q}, Q \left( S_{i}, P \right) \right), $ where $P$ is sampled from $ \mathcal{Q} $ and $h$ is sampled from $Q = Q \left( S_{i}, P \right) $ respectively.

The KL-divergence term is
\begin{equation}
\begin{aligned}
& D(\rho \| \pi) =\underset{f \sim \rho}{\mathbb{E}} \log \frac{\rho(f)}{\pi(f)} \\
&=\underset{P \sim \mathcal{Q}}{\mathbb{E}} \underset{h \sim Q(S, P)}{\mathbb{E}} \log \frac{\mathcal{Q}(P) Q\left(S_{i}, P\right)(h)}{\mathcal{P}(P) P(h)} \\
&=\underset{P \sim \mathcal{Q}}{\mathbb{E}} \log \frac{\mathcal{Q}(P)}{\mathcal{P}(P)}+\underset{P \sim \mathcal{Q}}{\mathbb{E}} \underset{h \sim Q(S, P)}{\mathbb{E}} \log \frac{Q\left(S_{i}, P\right)(h)}{P(h)} \\
&=D(\mathcal{Q} \| \mathcal{P})+\underset{P \sim \mathcal{Q}}{\mathbb{E}} D\left(Q\left(S_{i}, P\right) \| P\right).
\end{aligned}
\end{equation}

Just as with classical extended PAC-Bayes theory, our proof also involves two steps:

{\bf Step 1:} For the task $ i $, we use PAC-Bayes relative entropy bound to evaluate the generalization error for each observed task $ i $

\begin{eqnarray}\label{Proof_lambda_first_step}
\begin{array}{l}
\underset{P \sim \mathcal{Q}}{\mathbb{E}} er\left(Q, \mathcal{D}_{i} \right) \leq  \frac{1}{1-\lambda_i/2} \underset{P \sim \mathcal{Q}}{\mathbb{E}} \widehat{er} \left(Q,S_{i} \right) \\
\vspace{4pt}
+ \frac{ D(\mathcal{Q} \| \mathcal{P}) + \underset{P\sim\mathcal{Q}} {\mathbb{E}} D\left( Q \| P\right) + \log \frac{2\sqrt{m_i}}{\delta_{i}} }{ m_i \lambda (1-\lambda_i/2)}.
\end{array}
\end{eqnarray}

{\bf Step 2:} We try to bound the environment-level generalization. Due to observing only a finite number of tasks from the environment, re-using the classical PAC-Bayes bound yields

\begin{equation}
\begin{array}{l}
\underset{ (\mathcal{D}, m)  \sim \tau }{ \mathbb{E} }  \underset{  S \sim \mathcal{D}^{m} }{ \mathbb{E} } \underset{P \sim \mathcal{Q}}{\mathbb{E}} \underset{  h \sim Q(S,P) }{ \mathbb{E} } \underset{  z \sim \mathcal{D} }{ \mathbb{E} } \ell(h, z)  \\
\vspace{4pt}
\leq \frac{1}{n} \sum_{i=1}^{n} \underset{P \sim \mathcal{Q}}{\mathbb{E}} \underset{  h \sim Q(S_i ,P) }{ \mathbb{E} } \underset{  z \sim \mathcal{D}_i }{ \mathbb{E} } \ell(h, z) \\ 
\vspace{4pt}
+\sqrt{ \frac{1}{2(n-1)} \big( D( \mathcal{Q} \| \mathcal{P}) + \log \frac{n}{\delta_{0}} \big) }.
\end{array}
\end{equation}

For simplicity, we can rewrite the above formula as
\begin{equation}\label{Proof_lambda_second_step}
\begin{array}{l}
er( \mathcal{Q}, \tau ) \leq \frac{1}{n} \sum_{i=1}^{n} \underset{P \sim \mathcal{Q}}{\mathbb{E}} er(Q(S_i, P), \mathcal{D}_i ) \\
\vspace{4pt}
+ \sqrt{ \frac{1}{2(n-1)} \big( D( \mathcal{Q} \| \mathcal{P}) + \log \frac{n}{\delta_{0}} \big) }.
\end{array}
\end{equation}

Combining Eq. \ref{Proof_lambda_first_step} and Eq. \ref{Proof_lambda_second_step}  by union bound yields
\begin{equation}
\begin{array}{l}
er(\mathcal{Q}, \tau) \leq  \frac{1}{n} \sum_{i=1}^{n} \frac{1}{1-\lambda_i/2}  \underset{P \sim \mathcal{Q}}{\mathbb{E}} \hat{er} \left( Q, S_i \right)  \\
\vspace{4pt}
+ \frac{1}{n} \sum_{i=1}^{n} \frac{ D(\mathcal{Q} || \mathcal{P}) + \underset{P \sim \mathcal{Q}}{\mathbb{E}} D \left( Q || P\right) + \log \frac{2\sqrt{m_i}}{\delta_i} }{m_i\lambda_i (1-\lambda_i/2)}  \\
\vspace{4pt}
+ \sqrt{ \frac{1}{2(n-1)} \big( D ( \mathcal{Q} || \mathcal{P} ) + \log \frac{n}{\delta_0} \big) }.
\end{array}
\end{equation}

Assuming that  $ \delta_0 = \frac{\delta}{2} $, $ \delta_i = \frac{\delta}{2n} $ and $ \lambda_0 = \lambda_i = \lambda $, this then yields
\begin{equation}
\begin{array}{l}
er (\mathcal{Q}, \tau) \leq  \frac{1}{n} \sum_{i=1}^{n} \frac{1}{(1-\lambda/2)} \underset{P \sim \mathcal{Q}}{\mathbb{E}} \hat{er} \left(Q_{i}, S_i\right) \\
\vspace{4pt}
+\frac{1}{n} \sum_{i=1}^{n} \frac{ D(\mathcal{Q} || \mathcal{P}) + \underset{P \sim \mathcal{Q}}{\mathbb{E}} D\left(Q || P\right) + \log \frac{4n\sqrt{m_i}} {\delta} }{m_i\lambda (1-\lambda/2) } \\
\vspace{4pt}
+ \sqrt{ \frac{1}{2(n-1)} \left( D ( \mathcal{Q} || \mathcal{P} ) + \log \frac{2 n}{\delta} \right) }.
\end{array}
\end{equation}
\end{proof}

Furthermore, based on the Lemma \ref{PAC_Bayes_meta_classic} and Lemma \ref{PAC_Bayes_quadratic}, another PAC-Bayes meta-learning quadratic bound is proposed.

\begin{theorem}\label{Extend_PAC_quadratic}
(meta-learning PAC-Bayes quadratic bound). Let $ \mathcal{P} $ be some hyper-prior distribution, and $ Q $ be the posterior distribution, which is also called as the base learner. Then for any $\delta \in(0,1]$, the following inequality holds uniformly for all hyper-posteriors distributions $ \mathcal{Q} \in \mathcal{M} $ with probability at least $ 1-\delta $
\begin{equation}
\begin{array}{l}
 er (\mathcal{Q}, \tau) \leq \frac{1}{n} {\sum_{i=1}^{n}} \Big( \sqrt{ \underset{ P \sim \mathcal{Q} }{ \mathbb{E} } \hat{er} \big( Q, S_{i} \big) + \epsilon_i } + \sqrt{\epsilon_i} \Big)^2 \\
 + \sqrt{ \frac{1}{2(n-1)} \Big( D ( \mathcal{Q} || \mathcal{P} ) + \log \frac{2 n}{\delta} \Big) }.
\end{array}
\end{equation}
{\rm Here the} $ \epsilon_i = \frac{1}{2m_i} \big( D( \mathcal{Q} || \mathcal{P} ) + \underset{P \sim \mathcal{Q} }{ \mathbb{E} } D \left( Q || P \right) + \log \frac{4n \sqrt{m_i} }{\delta} \big)$.
\end{theorem}
% \begin{lemma}
% (PAC-Bayes quadratic bound for meta-learning). Under the same notations, for any $ \delta \in (0, 1] $, we have
% \begin{eqnarray}
% \begin{array}{c}
% er (h, \mathcal{D}) \leq \big( \sqrt{ \widehat{er} \left(h, S\right) + \varepsilon } + \sqrt{\varepsilon} \big)^2.
% \end{array}
% \end{eqnarray}
% {\rm Here we define} 
% \begin{eqnarray}
% \begin{array}{c}
% \varepsilon = \frac{ 1}{ 2m } \big( D\left(Q\|P\right) + \log\frac{2\sqrt{m}}{\delta} \big).
% \end{array}
% \end{eqnarray}
% \end{lemma}

\begin{proof}
For the task $ i $, applying PAC-Bayes relative entropy bound to evaluate the generalization error in each of the observed tasks $i$ yields
\begin{eqnarray}
\begin{array}{l}
\underset{P \sim \mathcal{Q}}{\mathbb{E}} er \left( Q, \mathcal{D}_i \right) \leq \Big( \sqrt{ \underset{P \sim \mathcal{Q}}{\mathbb{E}} \widehat{er} \left(Q, S_i \right) + \varepsilon_i } + \sqrt{\varepsilon_i} \Big)^2,
\end{array}
\end{eqnarray}
where $ \varepsilon_i = \frac{1}{2m_i} \Big( D(\mathcal{Q} || \mathcal{P}) + \underset{P \sim \mathcal{Q}}{\mathbb{E}} D\left(Q|| P\right) + \log\frac{2\sqrt{m_i}}{\delta_i} \Big)$.

By utilizing the first step of the proof in Theorem \ref{Extend_PAC_lambda}, and assuming that $ \delta_0 = \frac{\delta}{2} $, $ \delta_i = \frac{\delta}{2n} $, we can get the following meta-learning PAC-Bayes bound
\begin{eqnarray}
\begin{array}{l}
er (\mathcal{Q}, \tau)  \leq  \frac{1}{n} \sum_{i=1}^{n} \big( \sqrt{ \underset{ P \sim \mathcal{Q} }{ \mathbb{E} } \hat{er} \big( Q, S_{i} \big) + \epsilon_i } + \sqrt{\epsilon_i} \big)^2 \\
+ \sqrt{ \frac{1}{2(n-1)} \big( D ( \mathcal{Q} || \mathcal{P} ) + \log \frac{2 n}{\delta} \big) }.
\end{array}
\end{eqnarray}
Here $ \varepsilon_i = \frac{1}{m_i} \Big( D( \mathcal{Q} || \mathcal{P} ) + \underset{P \sim \mathcal{Q} }{ \mathbb{E} } D \left( Q  || P \right) + \log \frac{4n \sqrt{m_i} }{\delta} \Big) $.
\end{proof}

With different situations, different bounds may lead to different generalization performances. One can combine the two above-mentioned meta-learning bounds by a function which is defined piecewise to improve performance. The variational KL bound can take the minimum value of  Theorem \ref{Extend_PAC_lambda} and Theorem \ref{Extend_PAC_quadratic}, ensuring it is tight in both regimes. Prompted by PAC-Bayes variational bound, meta-learning PAC-Bayes variational bound is derived as follows:

\begin{theorem}\label{Extend_PAC_Variational}
(meta-learning PAC-Bayes variational bound). Let $P$ be some hyper-prior distribution, and $ Q $ be the posterior distribution, which is also known as base learner. Then for any $\delta \in(0,1]$, the following inequality holds uniformly for all hyper-posteriors distributions $ \mathcal{Q} \in \mathcal{M} $ with probability at least $ 1-\delta $
\begin{equation}
\begin{array}{l}
er (\mathcal{Q}, \tau)  \leq  \frac{1}{n} \sum_{i=1}^{n} \underset{ P \sim \mathcal{Q} }{ \mathbb{E} } \hat{e r}\left(Q, S_{i}\right) \\
 + \frac{1}{n} \sum_{i=1}^{n} {\rm min} \Big( \varepsilon_i + \sqrt{ \varepsilon_i \big( \varepsilon_i + 2 \underset{ P \sim \mathcal{Q} }{ \mathbb{E} } \hat{e r} \left( Q, S_i \right) \big)} , \sqrt{ \frac{\varepsilon_i}{2} } \Big) \vspace{2pt}\\
 + \sqrt{ \frac{1}{2(n-1)} \left( D ( \mathcal{Q} || \mathcal{P} ) + \log \frac{2 n}{\delta} \right) }.
\end{array}
\end{equation}
{\rm Here} $ \varepsilon_i = \frac{1}{m_i} \Big( D( \mathcal{Q} || \mathcal{P} ) + \underset{P \sim \mathcal{Q} }{ \mathbb{E} } D \left( Q || P \right) + \log \frac{4n \sqrt{m_i} }{\delta} \Big)$.
\end{theorem}

\begin{proof}
For the task $ i $, we use PAC-Bayes variational bound to estimate the generalization error in each of the observed tasks $ i $

\begin{eqnarray}
\begin{array}{l}
\underset{P \sim \mathcal{Q}}{\mathbb{E}} er \left( Q_i, \mathcal{D}_i \right) \leq \frac{1}{n} \sum_{i=1}^{n} \underset{P \sim \mathcal{Q}}{\mathbb{E}} \widehat{er} \left( Q, S_i \right) \\
+ \frac{1}{n} \sum_{i=1}^{n} {\rm min} \Big( \varepsilon_i + \sqrt{ \varepsilon_i \big( \varepsilon_i + 2 \underset{ P \sim \mathcal{Q} }{ \mathbb{E} } \hat{er} \left( Q, S_i \right) \big)} , \sqrt{ \frac{\varepsilon_i}{2} } \Big),
\end{array}
\end{eqnarray}
where $ \varepsilon_i = \frac{1}{2m_i} \Big( D(\mathcal{Q} || \mathcal{P}) + \underset{P \sim \mathcal{Q}}{\mathbb{E}} D\left(Q|| P\right) + \log\frac{2\sqrt{m_i}}{\delta_i} \Big)$.

By applying the first step of the proof in Theorem \ref{Extend_PAC_lambda}, and assuming that $ \delta_0 = \frac{\delta}{2} $, $ \delta_i = \frac{\delta}{2n} $ yields
\begin{eqnarray}
\begin{array}{l}
er(\mathcal{Q}, \tau)  \leq  \frac{1}{n} \sum_{i=1}^{n} \underset{ P \sim \mathcal{Q} }{ \mathbb{E} } \hat{e r}\left(Q_{i}\left(S_{i}, P\right), S_{i}\right)\\
+ \frac{1}{n} \sum_{i=1}^{n} {\rm min} \Big( \varepsilon_i + \sqrt{ \varepsilon_i \big( \varepsilon_i + 2 \underset{ P \sim \mathcal{Q} }{ \mathbb{E} } \hat{e r} \left( Q_i, S_i \right) \big)} , \sqrt{ \frac{\varepsilon_i}{2} } \Big) \\
\Big. + \sqrt{ \frac{1}{2(n-1)} \left( D ( \mathcal{Q} || \mathcal{P} ) + \log \frac{2 n}{\delta} \right) }.
\end{array}
\end{eqnarray}
Here $ \varepsilon_i = \frac{1}{m_i} \Big( D( \mathcal{Q} || \mathcal{P} ) + \underset{P \sim \mathcal{Q} }{ \mathbb{E} } D \left( Q || P \right) + \log \frac{4n \sqrt{m_i} }{\delta} \Big) $.
\end{proof}

\section{PAC-Bayes bounds with data-dependent prior}\label{section_5_bound_dp}
For the PAC-Bayes bound theory, the generalization error upper bound mainly depends mainly on the regularization item involving the distance between prior distribution $P$ and posterior distribution $Q$. However, the prior distribution is chosen randomly, with a view to measuring the parameter space. Especially in meta-learning, the generalization error bound involves both hyper-prior and hyper-posterior distributions, which are hard to converge. Seeking to solve this issue, in this section we aim to learn a localized prior distribution through the {\it ERM approach} on a part of the training samples. Then the remaining data will be used to optimize generalization error bound.

% \subsection{Classical PAC-Bayes bound with data-dependent prior}
% We first divide the training datasets $S$ into two separate dataset, one of the dataset $R$ is used to learn the mean of the prior distribution by ERM, and then the remaining dataset $S\backslash R$ is applied to calculate the PAC-Bayes bound.

% The classical single-task PAC-Bayes bound with data-dependent prior is 
% \begin{theorem}
% (McAllester's single-task bound with data-dependent prior). Let $ P \in \mathcal{M} $ be some prior distribution over $\mathcal{H}$. Then for any $ \delta \in(0,1] $, the following inequality holds uniformly for all posteriors distributions $ Q \in \mathcal{M} $ with probability at least $ 1 - \delta $
% \begin{eqnarray}
% \begin{array}{cc}
% er (Q, \mathcal{D}) \leq \widehat{er}(Q, S \backslash R) + \sqrt{ \frac{ \mathrm{KL}(Q\|P) + \log \frac{m}{\delta}} {2(m-1)} }
% \end{array}
% \end{eqnarray}
% \end{theorem} 
% Here, $P$ is the learnt prior distribution based on dataset  $ R $, and then $ S \backslash R $ is applied to calculate the bound.

% \subsection{Extended PAC-Bayes bound with data-dependent prior}
Akin to the classical PAC-Bayes bound with data-dependent prior, for the meta-learning, we try to propose a novel extended PAC-Bayes bound with data-dependent prior. Specifically, during the {\it training phase} of meta-learning, the corresponding dataset of training task $i$ is also divided into two separate datasets. Based on the {\it ERM approach}, one can learn a data-dependent prior distribution over a section of the training samples
\begin{equation}\label{Eq_ERM_approach}
\begin{array}{l}
P_{\theta} = \arg \min \underset{h \in P_{\theta} }{ \mathbb{E}} {er}_{\mathrm{emp}}(h, S),
\end{array}
\end{equation}
where ${er}_{\mathrm{emp}}(h)=\frac{1}{n} \sum_{i=1}^{n} \ell \left(h, R_i\right)$, $n$ is the number of all training tasks, $R_i$ is the sample subset of task $i$ selected from the whole dataset $S_i$, providing information which can be used to calculate data-dependent prior. The remaining data $S_i\backslash R_i$ of task $i$ is applied to evaluate the generalization error bound for meta-learning. In practice, expectations over distribution $P$ are difficult to calculate. Therefore, the Monte Carlo method deemed most effective in obtaining the numerical results of (\ref{Eq_ERM_approach}). Furthermore, prior distribution $P$ is selected directly from hyper posterior distribution $\mathcal{Q}_{\theta}$. So the learned parameters $P_{\theta}$ is designated as the initial mean parameter of $\mathcal{Q}_{\theta}$. The extended PAC-Bayes bound with data-dependent prior can then be shown as follows:

% Here, $Q_{D} \triangleq Q \left( S_i \backslash R_i, P \right) $ represents the posterior for task $i$ with data-dependent prior.

%  one of task set and its corresponding sample dataset are applied to learn a hyper-posterior distribution $ \mathcal{Q}_{D} $ and prior distribution $P$. Although in the training phase we try to optimize the hyper-posterior $\mathcal{Q_{\theta}}$, in meta-learning hyper-posterior $\mathcal{Q_{\theta}}$ is a distribution over $\mathcal{M}(P)$. Therefore, for convenience we consider this scenario as data-dependent prior.

% Note that all $n$ training data are used by the learning algorithm ($n_0$ examples used to build the prior, $n$ to learn the posterior and $ n - n_0 $ to evaluate the risk certificate).

\begin{theorem}
(Meta-learning PAC-Bayes bound with data-dependent prior). Let Q : $ \mathcal{Z}^{m} \times \mathcal{M} \rightarrow \mathcal{M} $ be a base learner, and let $\mathcal{P}$ be some predefined hyper-prior distribution. Then for any $ \delta \in (0,1] $ the following inequality holds uniformly for all hyper-posterior distributions $\mathcal{Q}$ with probability at least $ 1 - \delta $,
\begin{equation}
\begin{array}{c}
er (\mathcal{Q}, \tau) \leq \frac{1}{n} \sum_{i=1}^{n} \underset{P \sim \mathcal{Q}_{D}}{\mathbb{E}} \hat{er} \left( Q, S_i \backslash R_i \right) +  \\
+ \frac{1}{n} {\displaystyle\sum_{i=1}^{n}} \sqrt{ \frac{ D (\mathcal{Q}_{D} || \mathcal{P}) + \underset{ P \sim \mathcal{Q}_{D} }{\mathbb{E}} D \left( Q(P,S_i \backslash R_i ) || P\right) + \log \frac{2 {n} m_{i}}{\delta} }{ 2\left(m_{i}-1\right)} } \vspace{2pt} \\
+ \sqrt{\frac{ D(\mathcal{Q}_{D} || \mathcal{P}) + \log \frac{2n}{\delta} }{ 2(n-1) } }.
\end{array}
\end{equation}
\end{theorem}

Similarly, in the testing phase, when encountering new tasks,  one can also learn a data-dependent prior $P$ through the {\it ERM approach}.

\section{Practical meta-learning PAC-Bayes methods}\label{section_6_practical}
meta-learning PAC-Bayes bound tries to provide a generalization performance guarantee for the learned model with an arbitrarily high probability. In practice, one can train a probability neural network by minimizing the generalization error upper bound.

% This choice can be justified on the grounds that $ \ell^{\mathrm{x} - \mathrm{e}} (z, y) $ gives an upper bound on the probability of mistake when the label is chosen at random from the distribution produced by applying soft-max on $z$ (e.g., the output of the last linear layer of a neural network). 
\subsection{Training objectives}
Based on the proposed three meta-learning PAC-Bayes bounds, the corresponding three training objectives are developed as:

Theorem. \ref{Extend_PAC_lambda} and Theorem. \ref{Extend_PAC_quadratic} lead to the meta-learning PAC-Bayes $\lambda$ objective 
\begin{equation}\label{Eq_train_obj_lambda}
\begin{array}{ll}
f_{\lambda} (\theta) = \frac{1}{n} \sum_{i=1}^{n} \frac{1}{ (1 - \lambda_0 /2)^2  } \underset{ P \sim \mathcal{Q} }{ \mathbb{E} } \hat{er} \left( Q, S_i \right)\\
+\frac{1}{n} \sum_{i=1}^{n}  \frac{ D( \mathcal{Q} || \mathcal{P} ) + \underset{P \sim \mathcal{Q} }{ \mathbb{E} } D\left( Q || P \right) + \log \frac{ 4n\sqrt{m_i} }{ \delta } }{ m_i\lambda(1-\lambda /2)^2 } \vspace{4pt}\\
+ \sqrt{ \frac{1}{2(n-1)} \left( D ( \mathcal{Q} || \mathcal{P} ) + \log \frac{2 n}{\delta} \right) },
\end{array}
\end{equation}

% \operatorname{er} (Q, \mathcal{D}) \leq \widehat{er}(Q, S)
and meta-learning quadratic PAC-Bayes objective
\begin{equation}\label{Eq_train_obj_quad}
\begin{array}{l}
f_{\rm quad} (\theta) = \frac{1}{n} {\sum_{i=1}^{n}} \Big( \sqrt{ \underset{ P \sim \mathcal{Q} }{ \mathbb{E} } \hat{er} \big( Q, S_{i} \big) + \epsilon_i } + \sqrt{\epsilon_i} \Big)^2 \\
+ \sqrt{ \frac{1}{2(n-1)} \Big( D ( \mathcal{Q} || \mathcal{P} ) + \log \frac{2 n}{\delta} \Big) }.
\end{array}
\end{equation}
{\rm Here the} $ \epsilon_i = \frac{1}{2m_i} \big( D( \mathcal{Q} || \mathcal{P} ) + \underset{P \sim \mathcal{Q} }{ \mathbb{E} } D \left( Q || P \right) + \log \frac{4n \sqrt{m_i} }{\delta} \big)$.

By comparison, the training objective from Theorem. \ref{Extend_PAC_Variational} takes the following form
\begin{equation}\label{Eq_train_obj_varia}
\begin{array}{l}
f_{\rm varia} (\theta) = \frac{1}{n} \sum_{i=1}^{n} \underset{ P \sim \mathcal{Q} }{ \mathbb{E} } \hat{e r}\left(Q, S_{i}\right) \\
+ \frac{1}{n} \sum_{i=1}^{n} {\rm min} \Big( \varepsilon_i + \sqrt{ \varepsilon_i \big( \varepsilon_i + 2 \underset{ P \sim \mathcal{Q} }{ \mathbb{E} } \hat{e r} \left( Q, S_i \right) \big)} , \sqrt{ \frac{\varepsilon_i}{2} } \Big) \vspace{2pt}\\
+ \sqrt{ \frac{1}{2(n-1)} \left( D ( \mathcal{Q} || \mathcal{P} ) + \log \frac{2 n}{\delta} \right) },
\end{array}
\end{equation}
{\rm where the} $ \varepsilon_i = \frac{1}{m_i} \Big( D( \mathcal{Q} || \mathcal{P} ) + \underset{P \sim \mathcal{Q} }{ \mathbb{E} } D \left( Q || P \right) + \log \frac{4n \sqrt{m_i} }{\delta} \Big)$.

%  Then, based the previous two PAC-Bayes bound, we further propose two extended PAC-Bayes bound for meta-learning in the next step.

\subsection{Loss function}

As a rule, the standard loss function used on multi-class classification problems is the {\it cross-entropy loss function } $ \ell: \mathbb{R}^{k} \times[k] \rightarrow \mathbb{R} $ defined by 
\begin{equation}
\begin{array}{l}
\ell_{CE} = - \sum^{k}_{c=1} y_{o,c} \log(p_{o,c}),
\end{array}
\end{equation}
where $y_{o,c}$ is the binary indicator (0 or 1) if class label $c$ is the correct classification for observation $o$, $k$ is the number of class and $p_{o,c}$ is the predicted probability observation $o$ of class $c$. It is obviously that this loss function is unbounded loss function. However, the proposed meta-learning PAC-Bayes bound is only available for bounded loss function. Here, a ``bounded cross-entropy'' loss function is applied (See \cite{perez2020tighter}) as the surrogate loss for training in all experiments with $f_{\lambda}$, $f_{\rm quad}$ and $f_{\rm varia}$. Specifically, the loss function is clipped to $[0, \log(\frac{1}{p_{\rm min}}) ]$, where $p_{\rm min}$ is the lower bound of the network probabilities.
 
% where $\hat{l}_{S}^{\mathrm{x}-\mathrm{e}}(w)=\frac{1}{n} \sum_{i=1}^{n} \tilde{\ell}_{1}^{\mathrm{x}-\mathrm{e}}\left(h_{w}\left(X_{i}\right), Y_{i}\right)$ denotes the empirical error rate for the cross-entropy loss, and $h_{w}: \mathcal{X} \rightarrow \mathbb{R}^{k}$ denotes the function implemented by the neural network that uses weights $w$.

\begin{algorithm}[tbp]
\caption{Meta training phase, without data-dependent prior}
\label{alg_meta_train}
  \begin{algorithmic}[1]
    \Require
        Datasets of $n$ training tasks: $S_1, ..., S_n$.
    \Ensure
        Learned meta-learner with parameter $\theta$.
    \State Initializing hyper-prior $\mathcal{P}$, hyper-posterior $\mathcal{Q}$, prior model $\theta$, posterior model $\phi_i, i=1,...,n$;
    \While{not done}
        \For{task $i, i=1,...,n$}
            \State Sample mini-batch from datasets $S_i, i=1,...,n$
            \State Calculate $D(Q_{\phi_i}\| P_{\theta})$(\ref{Eq_D_qp})
            \State Calculate 
            $\underset{ P \sim \mathcal{Q} }{ \mathbb{E} } \hat{er} \left( Q, S_i \right)$ by Monte-Carlo method
        \EndFor
        \State Compute the training objective $f$ (see \ref{Eq_train_obj_lambda}, \ref{Eq_train_obj_quad} or \ref{Eq_train_obj_varia})
        \State Gradient step using $\left[ \begin{array}{c} \nabla_{\theta} f \\ \nabla_{\phi_i} f \end{array} \right] $
    \EndWhile\\
    % \label{code:from:select} \\
    \Return $\theta$;
  \end{algorithmic}
\end{algorithm}

% We explain now the differences between the relaxed and refined versions of the Pinsker inequality, used for defining the above presented PAC-Bayes inspired training objectives and crucial to understand their differences. Recall the definition of binary KL divergence:
\subsection{Gaussian weight distributions}
% In order to achieve better performance, we investigate the weights with different prior distributions: Gaussian weight distributions and Laplace weight distributions.
In this section, the specific forms of hyper-prior distribution $\mathcal{P}$, hyper-posterior distribution $\mathcal{Q}$ and weights distribution of the stochastic neural network are selected. 

For the meta-learning PAC-Bayes bound, the hyper-prior distribution $\mathcal{P}$ is set as a zero-mean Gaussian distribution
\begin{equation}
\begin{array}{l}
\mathcal{P} \triangleq \mathcal{N}\left(0, \kappa_{\mathcal{P}}^{2} I_{N_{P} \times N_{P}}\right),
\end{array}
\end{equation}
where $\kappa_{\mathcal{P}}>0$ is constant and $N_{P}$ is the number of neural network parameters $w$. 

Correspondingly, the hyper-posterior distribution $\mathcal{Q}$, which consists of all distributions over $\mathbb{R}^{N_P}$, is defined as a family of isotropic Gaussian distributions as follows
\begin{equation}
\begin{array}{l}
\mathcal{Q}_{\theta} \triangleq \mathcal{N}\left(\theta, \kappa_{\mathcal{Q}}^{2} I_{N_{P} \times N_{P}}\right),
\end{array}
\end{equation}
where $\kappa_{\mathcal{Q}}>0$ is also a predefined constant. Therefore the KL divergence between the hyper-prior distribution $\mathcal{P}$ and hyper-posterior distribution $\mathcal{Q}$ equals
\begin{eqnarray}
\begin{array}{cc}
D \left( \mathcal{Q}_{\theta} \| \mathcal{P} \right) = \frac{ \|\theta\|_{2}^{2} + \kappa_{\mathrm{Q}}^{2} }{ 2 \kappa_{\mathcal{P}}^{2} } + \log \frac{ \kappa_{\mathcal{P}} }{ \kappa_{\mathcal{Q}} } - \frac{1}{2}.
\end{array}
\end{eqnarray}

In the PAC-Bound theory, a probability  neural  network is applied, which means all weights $w$ are stochastic variables drawing from prior or posterior distribution. In this paper, we define that each weight $w_i$ in the neural network as it obeys Gaussian distribution. The prior $P_{\theta}$ and the posteriors $Q_{\phi_{i}}, i=1, \ldots, n,$ are defined as factorized Gaussian distributions

% given a parametric family of priors $\left\{P_{\tilde{\theta}}: \tilde{\theta} \in \mathbb{R}^{N_{P}}\right\}$, $N_{P} \in \mathbb{N},$ the space of hyper-posteriors consists of all distributions over $\mathbb{R}^{N_{P}}$.  Notice that $\mathcal{Q}$ appears in the extended bound in two forms (i) divergence from the hyper-prior $D(\mathcal{Q} \| \mathcal{P})$ and (ii) expectations over $P \sim \mathcal{Q}$.

% Note that the hyper-prior acts as a regularization term which prefers solutions with small $L_{2}$ norm.
\begin{eqnarray}
\begin{array}{cc}
P_{\theta}(w) = \prod_{k=1}^{d} \mathcal{N} \left( w_{k} ; \mu_{P, k}, \sigma_{P, k}^{2} \right),
\end{array}
\end{eqnarray}
\begin{eqnarray}
\begin{array}{cc}
Q_{\phi_{i}}(w) = \prod_{k=1}^{d} \mathcal{N}\left(w_{k} ; \mu_{i, k}, \sigma_{i, k}^{2} \right),
\end{array}
\end{eqnarray}
where $d$ is the number of neural network parameters and $n$ is the number of tasks. The corresponding KL divergence between prior $P_{\theta}$ and the posteriors $Q_{\phi_{i}}, i=1, \ldots, n,$ is 
\begin{eqnarray}\label{Eq_D_qp}
\begin{array}{l}
D \left( Q_{\phi_{i}} \| P_{\theta} \right) = \\ \frac{1}{2} \sum_{k=1}^{d} \Big(\log \frac{ \sigma_{P, k}^{2} }{ \sigma_{i, k}^{2} } + \frac{ \sigma_{i, k}^{2} + \left( \mu_{i, k} - \mu_{P, k} \right)^{2} }{ \sigma_{P, k}^{2} } - 1 \Big).
% \log \frac{ \sigma_{P, k}^{2} }{ \sigma_{i, k}^{2} } +  \frac{ \sigma_{i, k}^{2} + ( \mu_{i, k} - \mu_{P, k} )^2 }{ \sigma_{P, k}^{2} } - 1.
\end{array}
\end{eqnarray}

As we started earlier, the prior distribution $P_{\tilde{\theta}}$ is sampled from hyper-posterior distribution $\mathcal{Q}_\theta$. Practically, it follows that the prior distribution parameters $\tilde{\theta} = \theta + \varepsilon_{P}, \varepsilon_{P} \sim \mathcal{N} \left(0, \kappa_{\mathcal{Q}}^{2} I_{N_{P} \times N_{P}} \right) $. In other words, prior distribution $P_{\tilde{\theta}}$ sampling from hyper-posterior distribution $\mathcal{Q}_\theta$ means adding Gaussian noise $\varepsilon_{P}$ to the parameters $\theta$ during training. The specific pseudo code is shown in Algorithm \ref{alg_meta_train} and Algorithm \ref{alg_meta_train_dp} for both random prior and data-dependent prior respectively.

\begin{algorithm}[tbp]
\caption{Meta training phase, with data-dependent prior}
\label{alg_meta_train_dp}
  \begin{algorithmic}[1]
    \Require
        Datasets of $n$ training tasks: $S_1, ..., S_n$.
    \Ensure
        Learned meta-learner with parameter $\theta$.
    \State Initializing  prior model $\theta$;
    \State Separate training datasets $S_i$ into two parts $S_i/R_i, R_i, i=1,...,n$;
    \While{not done}
        \State Sample mini-batch from datasets $R_i, i=1,...,n$
        \State Calculate $\underset{h \in P_{\theta} }{ \mathbb{E}} {er}_{\mathrm{emp}}(h, S)$ (\ref{Eq_ERM_approach})
        \State Gradient step using $ \nabla_{\theta} f $
    \EndWhile
    \State Initializing  hyper-posterior $\mathcal{Q}$ with learned parameter $\theta$, prior model $P$, posterior model $Q, i=1,...,n$;
    \While{not done}
        \For{task $i, i=1,...,n$}
            \State Sample mini-batch from  $S_i/R_i, i=1,...,n$
            \State Calculate $D(Q_{\phi_i}\| P_{\theta})$(\ref{Eq_D_qp})
            \State Calculate 
            $\underset{ P \sim \mathcal{Q} }{ \mathbb{E} } \hat{er} \left( Q, S_i \right)$ by Monte-Carlo method
        \EndFor
        \State Compute the training objective $f$ (see \ref{Eq_train_obj_lambda}, \ref{Eq_train_obj_quad} or \ref{Eq_train_obj_varia})
        \State Gradient step using $\left[ \begin{array}{c} \nabla_{\theta} f \\ \nabla_{\phi_i} f \end{array} \right] $
    \EndWhile\\
    % \label{code:fram:select} \\
    \Return $\theta$;
  \end{algorithmic}
\end{algorithm}

\section{Experiments}\label{section_7_experiments}
In this section, the performance of our proposed meta-learning PAC-Bayes bound algorithms is illustrated with image classification tasks solved by stochastic neural networks. Specifically, we conduct our procedure within two different environments based on the MNIST dataset, those being {\it permuted pixels} and {\it permuted labels}. For the {\it permuted pixels} environment, each task is constructed by a shuffle of image pixels with 60000 training samples and 10000 testing samples. For the {\it permuted labels} environment, each task is generated by a permutation of image labels with the same number of training and testing samples as produced in the permuted pixels environment.

For the shuffled pixels experiment, the neural network structure selected is a full connected neural network (FCN) with 4 layers (3 hidden layers and a linear output layer) and 400 units per layer. For the permuted labels experiment, the neural network structure is designated as  a 4-layers convolutional neural network (CNN), comprising 2 convolution layers ,each with $5 \times 5$ kernels, and 2 full connected layers. For all experiments, ReLU activations are used. The optimizer is selected as Adam, with a learning rate of $10^{-3}$.

For both of two experiments, each initialized log-var $\log\sigma^{2}_{P}$ of weights is drawn from $\mathcal{N}(-10,0.01)$. The hyper-prior and hyper-posterior parameters are $ \kappa_{\mathcal{P}} = 2000$ and $\kappa_{\mathcal{Q}} = 0.001$ respectively. In the meta-learning PAC-Bayes bound, the confidence parameter chosen is $\delta=0.1$. Source code is available at GitHub\footnote{Codes are available on \url{https://github.com/tyliu22/Meta-learning-PAC-Bayes-bound-with-data-depedent-prior.git}}.

\subsection{Comparison of various PAC-Bayes bounds}
In this section, we focus on the performance of three proposed meta-learning PAC-Bayes bounds. For the meta-training phase, we run the total training of 50 epochs, maximal number of tasks in each meta-batch being 16, while 10 tasks are used for the meta-learner to learn. For the testing phase, we run a total testing of 20 epochs, using 20 tasks to confirm the  meta-learner performance. We select 128 as the data batch size for training and testing.

First, we investigate the weights of the stochastic neural network. As shown in Figure \ref{model_ana_pixel}, the average log-variance parameter of each layer's weights is analyzed. The higher the average $\log(\sigma^2)$ is, the more flexible are the weights are. In the shuffled pixel experiment, the lower layers perform with high variance which can extract the feature of shuffled-pixels image robustly, while the higher layers perform with a low variance which corresponds with fixed labels. Contrastingly, in the permuted label experiments, as shown in Figure \ref{model_ana_label}, the higher layers perform with high variance which can adapt robustly to the permutation of image label, and the lower layers' low variance performance corresponds with fixed-image samples.

Next, the influence of different numbers of training tasks on performance is analyzed, in relation to generalization error bound, empiric loss and empiric error. In Figure \ref{NumTrainTask_NewTask}, it is clear that, as the number of training tasks increases, the learned model achieves improved generalization performance and accuracy. 

\begin{figure}[tbp]
	\centering 
	\subfigure[]{
        \label{model_ana_pixel}
        \includegraphics[width=0.22\textwidth]{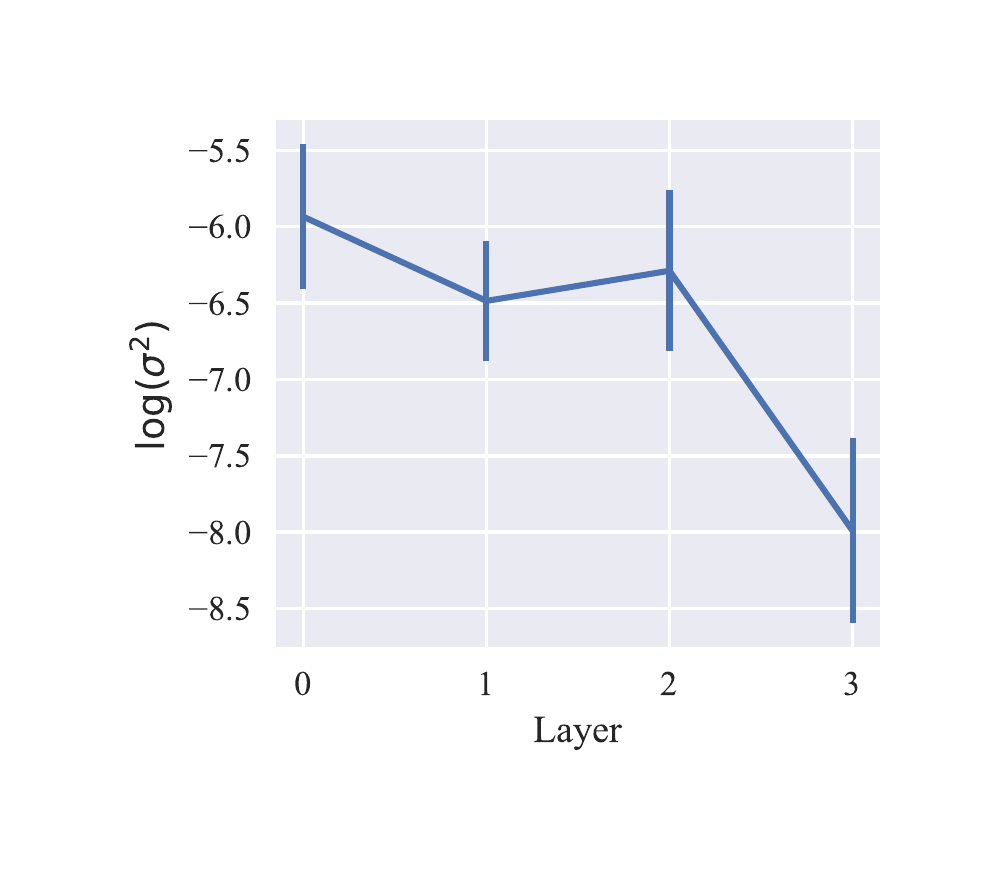}}
	\subfigure[]{
	    \label{model_ana_label}
        \includegraphics[width=0.22\textwidth]{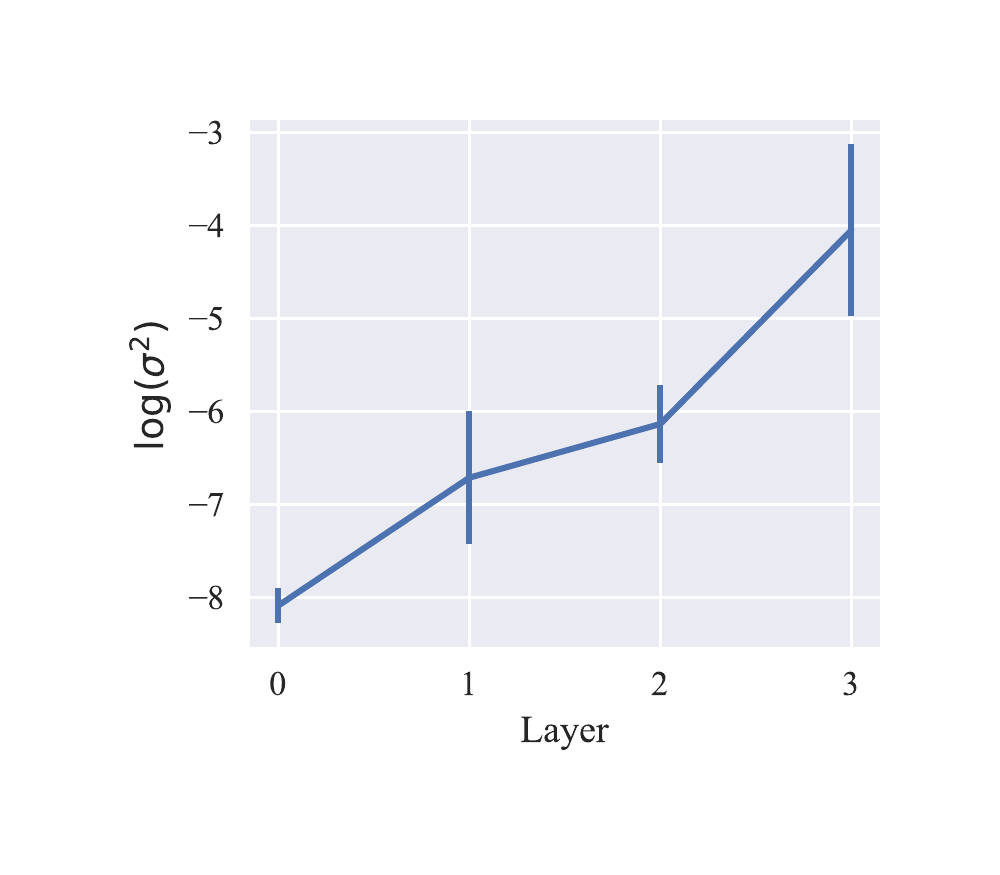}}
	\caption{Model parameter analysis by layers, $\log(\sigma^2)$ represents the weight uncertainty of each layer. (a) Prior model parameter analysis  for shuffled pixels environment; (b) Prior model parameter analysis for permuted labels environment.}
	\label{Prior_model_analysis}
\end{figure}

\begin{figure*}[!htbp]
	\centering 
	\subfigure[Generalization error bound]{
        \label{NumTrainTask_BoundonNewTask}
        \includegraphics[width=0.25\textwidth]{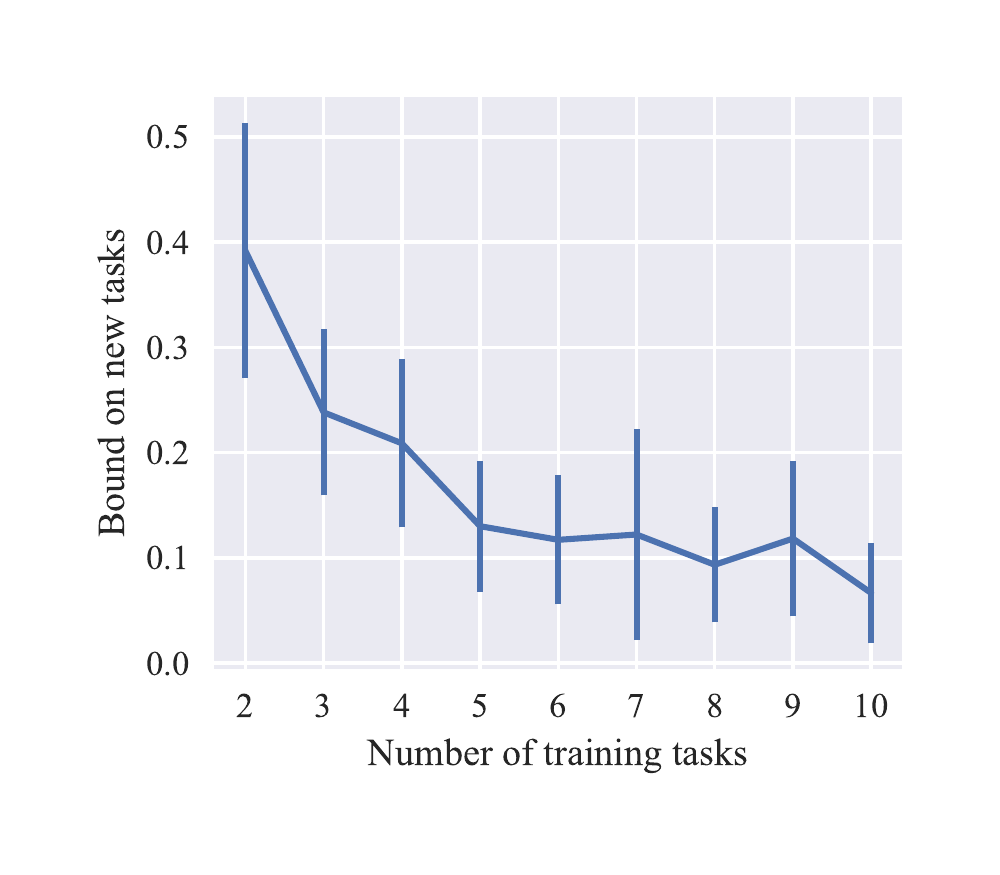}}
	\quad
	\subfigure[Empiric loss]{
	    \label{NumTrainTask_LossonNewTask}
        \includegraphics[width=0.27\textwidth]{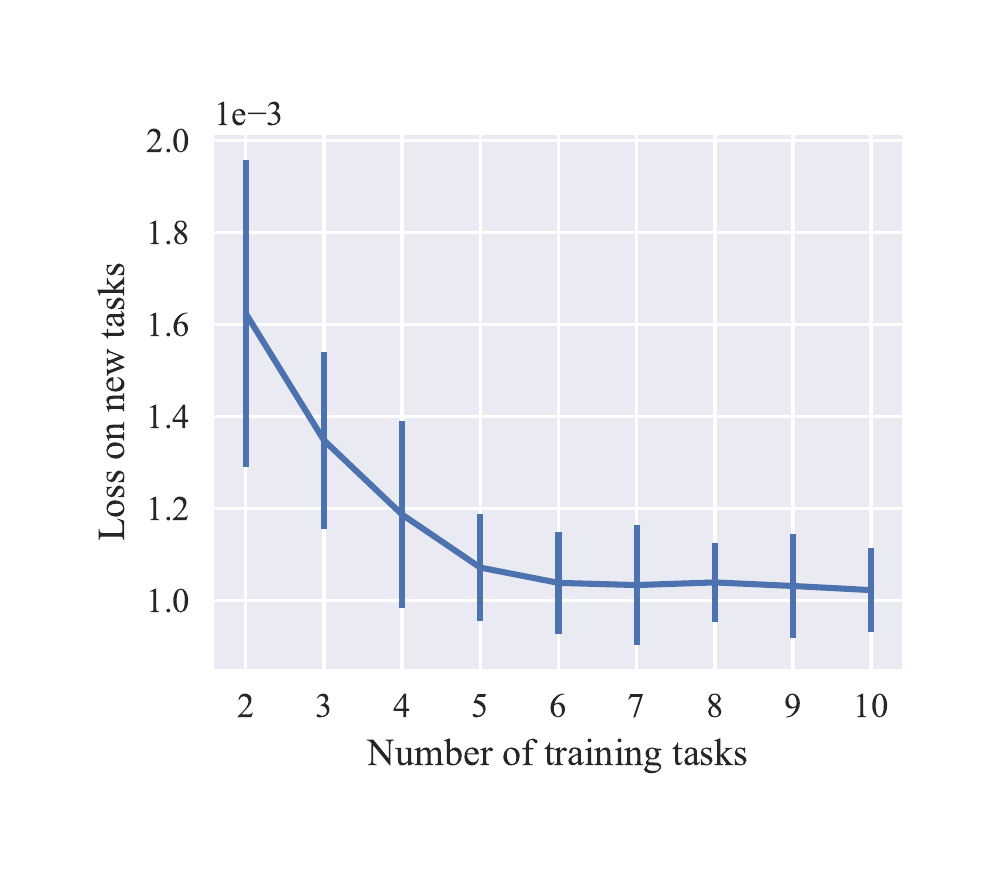}}
	\subfigure[Empiric error]{
        \label{NumTrainTask_ErrorNewTask}
        \includegraphics[width=0.25\textwidth]{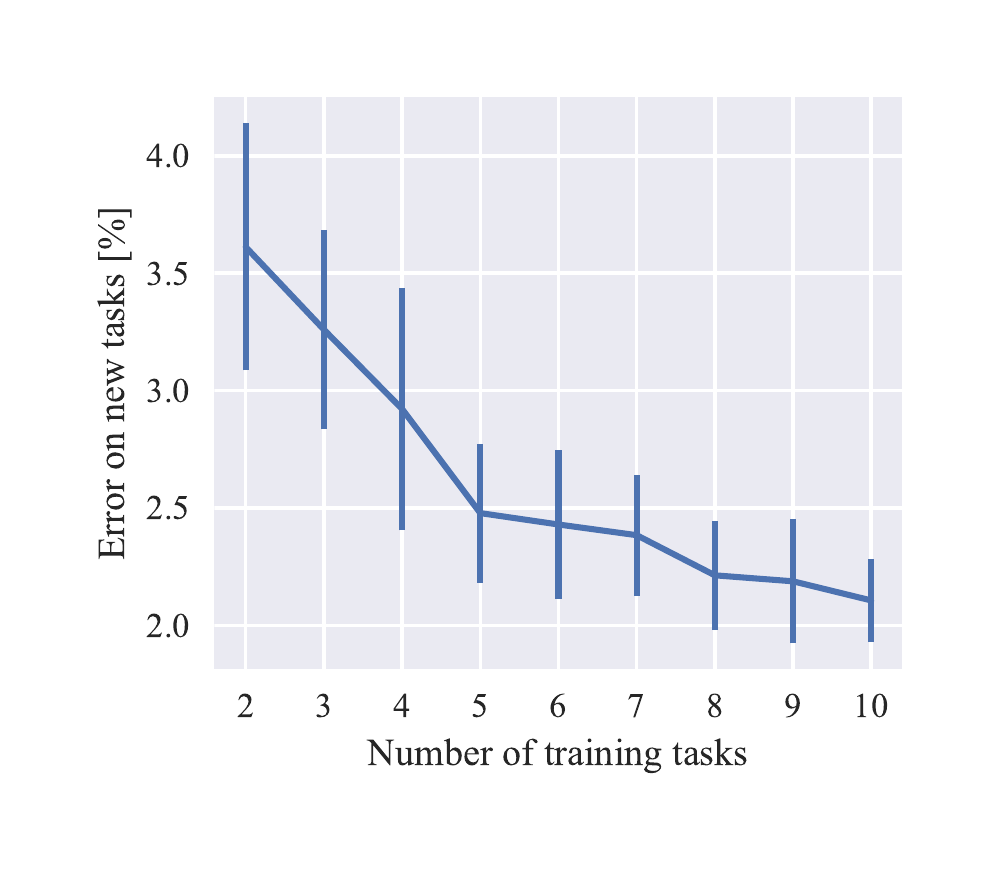}}
	\caption{The average performance of learning new tasks with different number of training tasks. (a) The average generalization error bound of learning new tasks; (b) The average empiric loss of learning new tasks; (c) The average empiric error of learning new tasks.}
	\label{NumTrainTask_NewTask}
\end{figure*}

\begin{table*}[!htbp]
\centering
\caption{Comparison of various PAC-Bayes bounds in training phase in both shuffled pixels and permuted labels environment with different prior model.}
\label{Table_analysis_train}
\begin{tabular}{cccccccc}
\toprule[1pt]
Environment & Prior model & Objective   & Bound  & Task complexity & Meta complexity & Empirical loss & Error ($\%$) \\ \hline
\multirow{10}{*}{Shuffled pixels}&\multirow{5}{*}{Random prior}   & $f_{\rm classic}$ & 0.8117 &   0.1368     & {\bf 0.6199}  & 0.05503   & 2.19 \\
& & $f_{\rm Seeger}$  & 0.7924 &   0.1312     & 0.6312     & 0.02993      & 2.02 \\
& & $f_{\lambda}$     & 0.8094 &   {\bf 0.09615} & 0.6327  & 0.04027      & 2.15 \\
& & $f_{\rm quad}$    & 0.7761 &   0.1205     & 0.6286     & 0.02696      & {\bf 2.01} \\
& & $f_{\rm varia}$   & {\bf 0.7738} & 0.1249 & 0.6262     & {\bf 0.02272} & {\bf 2.01} \\ \cline{2-8}
&\multirow{5}{*}{data-dependent prior}   & $f_{\rm classic}$  & 0.7213 &   0.01808    & 0.5519    & 0.1514   &  {\bf 4.02}  \\
& & $f_{\rm Seeger}$   & 0.7030 &   0.01213    & {\bf 0.5516}    & 0.1393   &  10.05  \\
& & $f_{\lambda}$      & 0.8095 &   {\bf 0.00223}    & 0.5547    & {\bf 0.1263}   &  6.81  \\
& & $f_{\rm quad}$     & {\bf 0.7006} &   0.01271    & 0.5523    & 0.1357   &  6.94  \\
& & $f_{\rm varia}$    & 0.7681 &   0.01503    & 0.5518    & 0.2012   &  5.22  \\ \hline
\multirow{10}{*}{Permuted labels}&\multirow{5}{*}{Random prior} &$f_{\rm classic}$  & 0.6476 &   0.07143    & {\bf 0.5504}    & 0.02574      &  0.813  \\
& &$f_{\rm Seeger}$   & 0.6463 &   0.06144    & 0.5513    & 0.03357      &  0.880  \\
& &$f_{\lambda}$      & 0.6258 &   {\bf 0.04255}    & 0.5513    & {\bf 0.01595}      &  {\bf 0.723}  \\
& &$f_{\rm quad}$     & 0.6360 &   0.05332    & 0.5514    & 0.03136      &  0.878  \\
& &$f_{\rm varia}$    & {\bf 0.6208} &   0.04597    & 0.5505    & 0.02419      &  0.811  \\ \cline{2-8}
&\multirow{5}{*}{data-dependent prior} &$f_{\rm classic}$  & 0.6519 &   0.08259   & {\bf 0.5505} & 0.01887  &  0.781  \\
& &$f_{\rm Seeger}$   & 0.6254 &   0.06053   & 0.5506     & 0.01422  &  0.771  \\
& &$f_{\lambda}$      & 0.6290 &   0.05074   & 0.5509     & {\bf 0.01369}  &  0.748  \\
& &$f_{\rm quad}$     & 0.6242 &   0.05539   & 0.5507     & 0.01816  &  {\bf 0.708}  \\
& &$f_{\rm varia}$    & {\bf 0.6143} & {\bf 0.04935}   & 0.5506     & 0.01440  &  0.781  \\
\bottomrule[1pt]
\end{tabular}
\vspace{2pt}

The performance in the training phase of five meta-learning training objectives on the shuffled pixels and permuted labels MNIST environments  with different prior models is analyzed in the above Table, in terms of PAC Bayes bound, task complexity, meta complexity, empirical loss and estimated  error.
\end{table*}

\begin{table*}[!htbp]
\centering
\caption {Comparison of various PAC-Bayes bounds in testing phase in both shuffled pixels and permuted labels environment with different prior model ($\pm$ indicates the $95\%$ confidence interval).}
\label{Table_analysis_test}
\begin{tabular}{cccccc}
\toprule[1pt]
Environment & Prior model &Training objective    & Test bound   & Test loss (${\rm e}-04$)  & Test error (\%) \\ \hline
\multirow{10}{*}{Shuffled pixels}&\multirow{5}{*}{Random prior}   &    $f_{\rm classic}$        &   $0.1580\pm 0.02075$ &  $\mathbf{8.861\pm 0.4593}$  & $2.541\pm 0.1341$    \\
& &     $f_{\rm Seeger}$       &   $0.2196\pm 0.04459$ &  $12.29\pm 0.9517$  & $2.832\pm 0.169$    \\
& & $f_{\lambda}$   &   $\mathbf{0.1271\pm 0.01617}$ &  $9.482\pm 0.5155$  & $2.538\pm 0.1224$   \\
& & $f_{\rm quad}$        &   $0.1957\pm 0.03627$ &  $10.61\pm 0.6422$  & $2.753\pm 0.009047$    \\
& & $f_{\rm varia}$       &   $0.1284\pm 0.02872$ &  $9.403\pm 0.3873$  & $\mathbf{2.432\pm 0.07485}$    \\  \cline{2-6}
&\multirow{5}{*}{data-dependent prior}   &    $f_{\rm classic}$        &   $0.1660\pm 0.05372$ &  $\mathbf{1.669\pm 0.2091}$  & $\mathbf{3.614\pm 0.3152}$    \\
& &    $f_{\rm Seeger}$      &   $0.1627\pm 0.04773$ &  $1.798\pm 0.2311$  & $4.033\pm 0.3035$    \\
& & $f_{\lambda}$   &   $0.1893\pm 0.01472$ &  $1.816\pm 0.2451$  &  $3.692\pm 0.4079$   \\
& & $f_{\rm quad}$        &   $0.1671\pm 0.08549$ &  $2.175\pm 0.4371$  & $3.719\pm 0.3026$    \\
& & $f_{\rm varia}$       &   $\mathbf{0.1610\pm 0.06316}$ &  $1.790\pm 0.2783$  & $3.809\pm 0.2905$    \\ \hline
\multirow{10}{*}{Permuted labels}&\multirow{5}{*}{Random prior} &    $f_{\rm classic}$        &   $\mathbf{2.905\pm 0.1256}$   &  $1.156\pm 0.09357$ & $42.62\pm 4.999$    \\
& &     $f_{\rm Seeger}$       &   $3.196\pm 0.1018$   &  $1.374\pm 0.06323$ & $54.83 \pm 4.557$    \\
& & $f_{\lambda}$   &   $3.282\pm 0.1580$   &  $\mathbf{1.074\pm 0.09365}$ & $\mathbf{40.73\pm 5.466}$    \\
& & $f_{\rm quad}$        &   $3.284\pm 0.1081$   &  $1.393\pm 0.06071$ & $55.60\pm 4.241$    \\
& & $f_{\rm varia}$       &   $3.410\pm 0.1423$   &  $1.241\pm 0.09571$ & $48.75\pm 5.178$    \\ \cline{2-6}
&\multirow{5}{*}{data-dependent prior} &    $f_{\rm classic}$       &   $0.1227\pm 0.02161$ &  $\mathbf{3.354\pm 0.1787}$  & $0.896\pm 0.04079$    \\
& &    $f_{\rm Seeger}$      &   $0.1220\pm 0.03072$ &  $3.426\pm 0.1126$  & $0.950 \pm 0.03477$    \\
& & $f_{\lambda}$   &   $\mathbf{0.1071\pm 0.03376}$ &  $3.643\pm 0.1649$  & $0.934\pm 0.03316$    \\
& & $f_{\rm quad}$        &   $0.1358\pm 0.03387$ &  $3.374\pm 0.1499$  & $0.913\pm 0.03068$    \\
& & $f_{\rm varia}$       &   $0.1099\pm 0.02147$ &  $3.370\pm 0.1427$  & $\mathbf{0.882\pm 0.02697}$    \\
\bottomrule[1pt]
\end{tabular}
\vspace{2pt}

The performance in the testing phase of five meta-learning training objectives on the shuffled pixels and permuted labels MNIST environments  with different prior models is analyzed in the above Table, in terms of test bound, test loss and test error.
\end{table*}

\begin{figure*}[!htbp]
	\centering 
	\subfigure[]{
        \label{Pixels_bound_training}
        \includegraphics[width=0.2\textwidth]{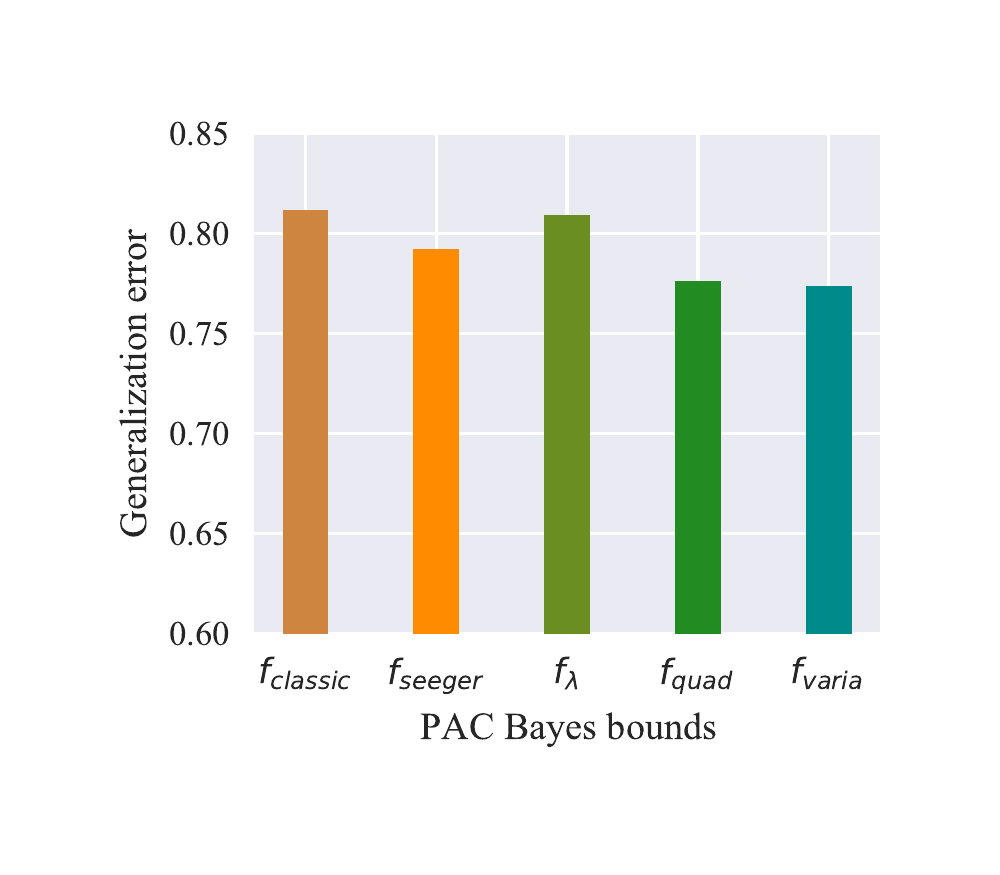}}
	\subfigure[]{
	    \label{Pixels_bound_testing}
        \includegraphics[width=0.2\textwidth]{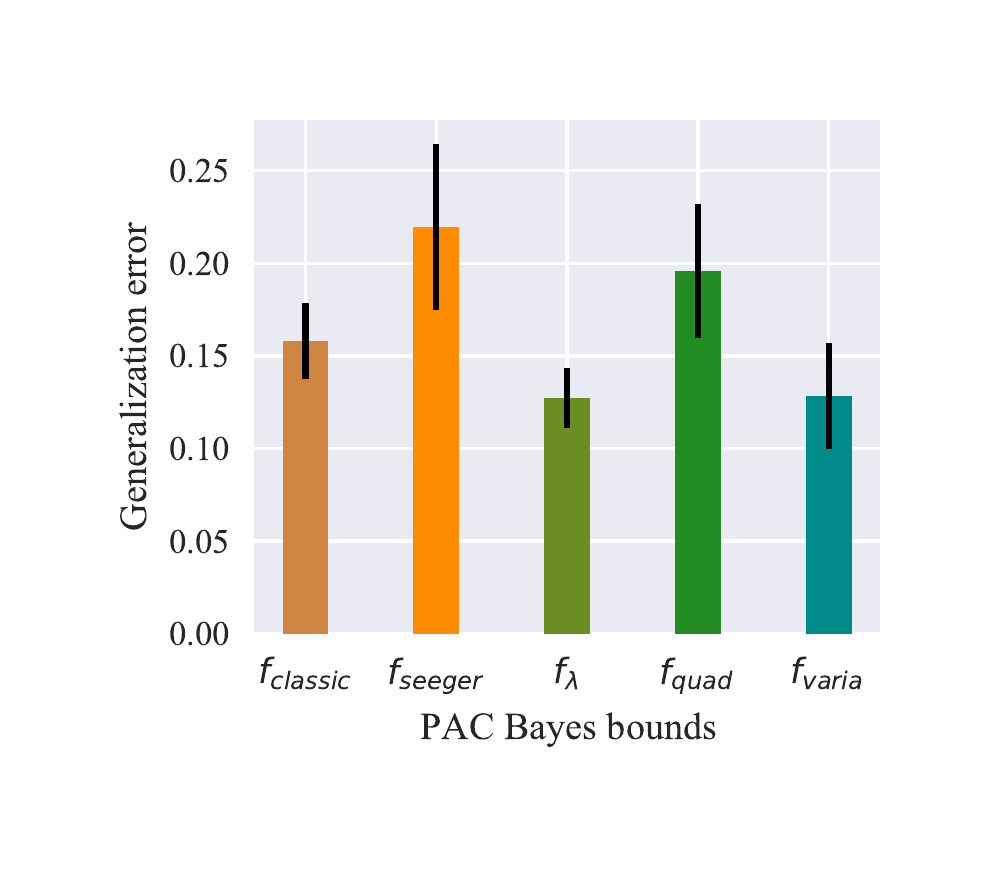}}
	\subfigure[]{
        \label{Labels_bound_training}
        \includegraphics[width=0.2\textwidth]{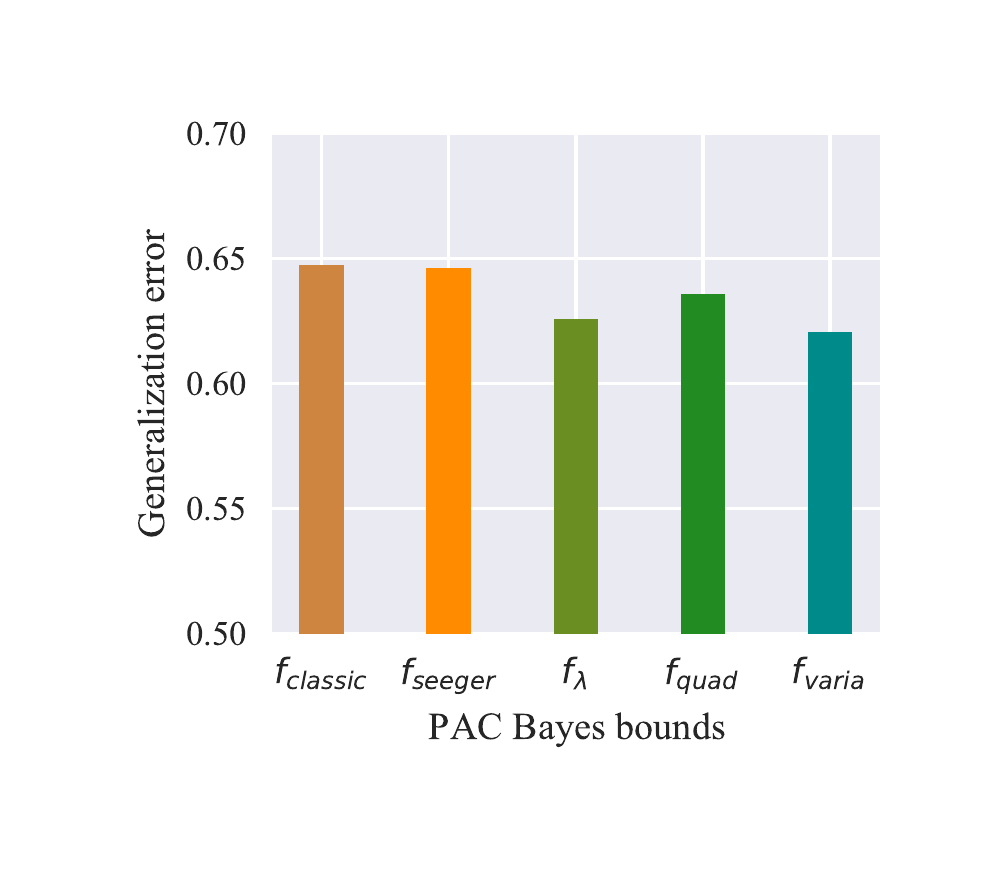}}
	\subfigure[]{
	    \label{Labels_bound_testing}
        \includegraphics[width=0.2\textwidth]{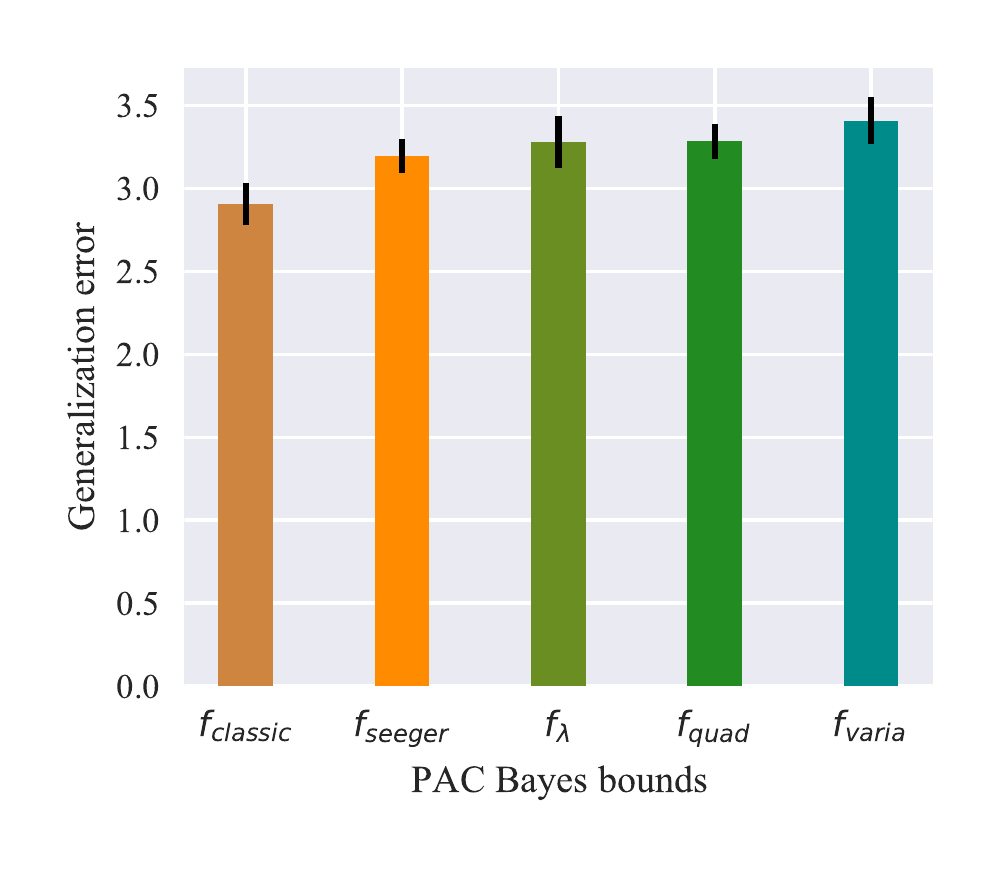}}
    \subfigure[]{
        \label{DP_Pixels_bound_training}
        \includegraphics[width=0.2\textwidth]{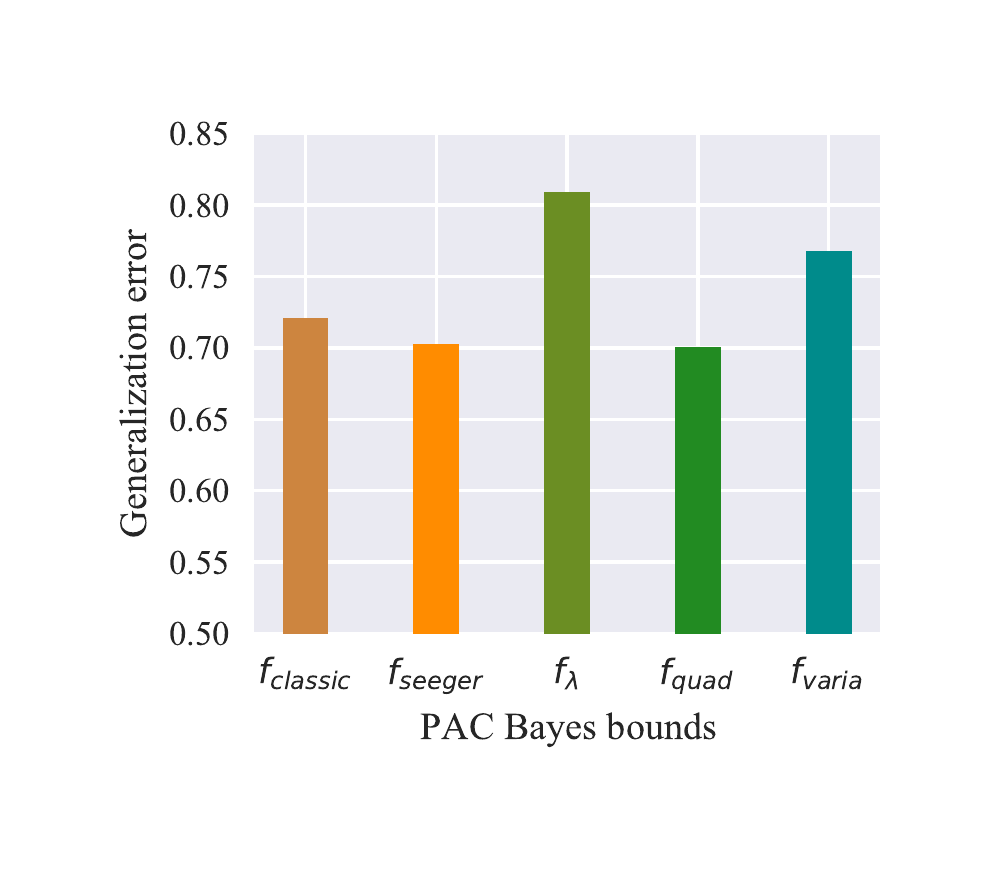}}
    \subfigure[]{
        \label{DP_Pixels_bound_testing}
        \includegraphics[width=0.2\textwidth]{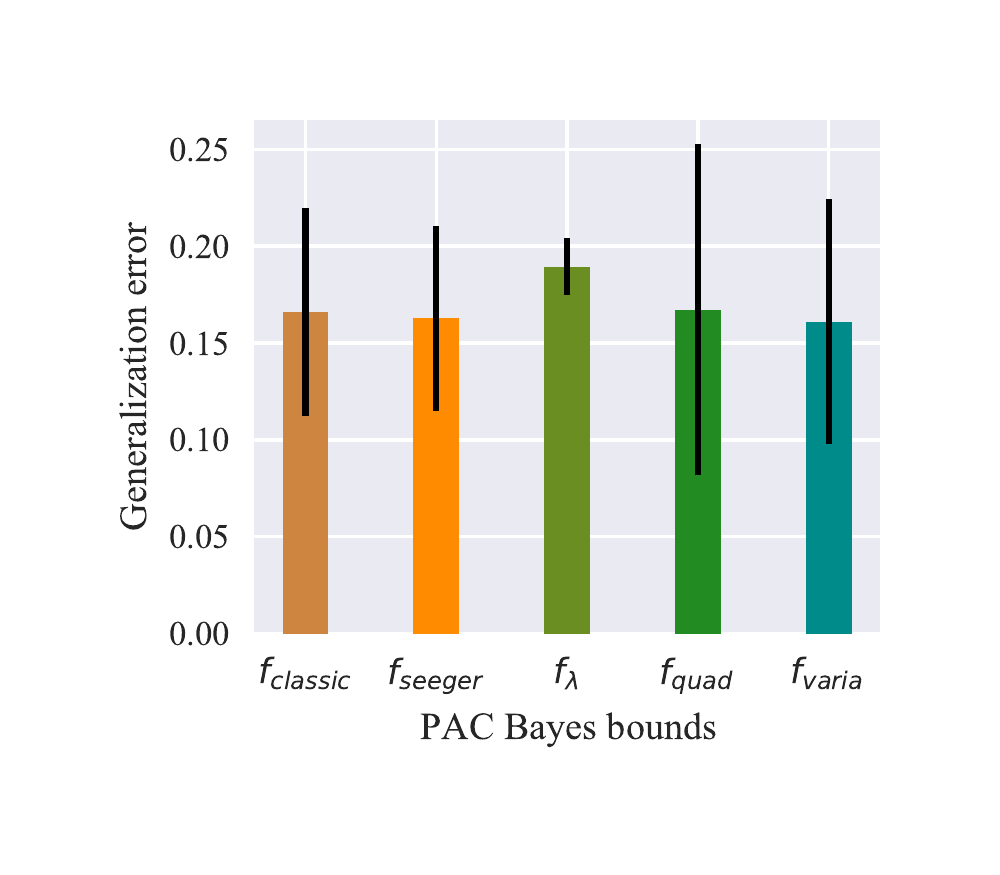}}
	\subfigure[]{
        \label{DP_Labels_bound_training}
        \includegraphics[width=0.2\textwidth]{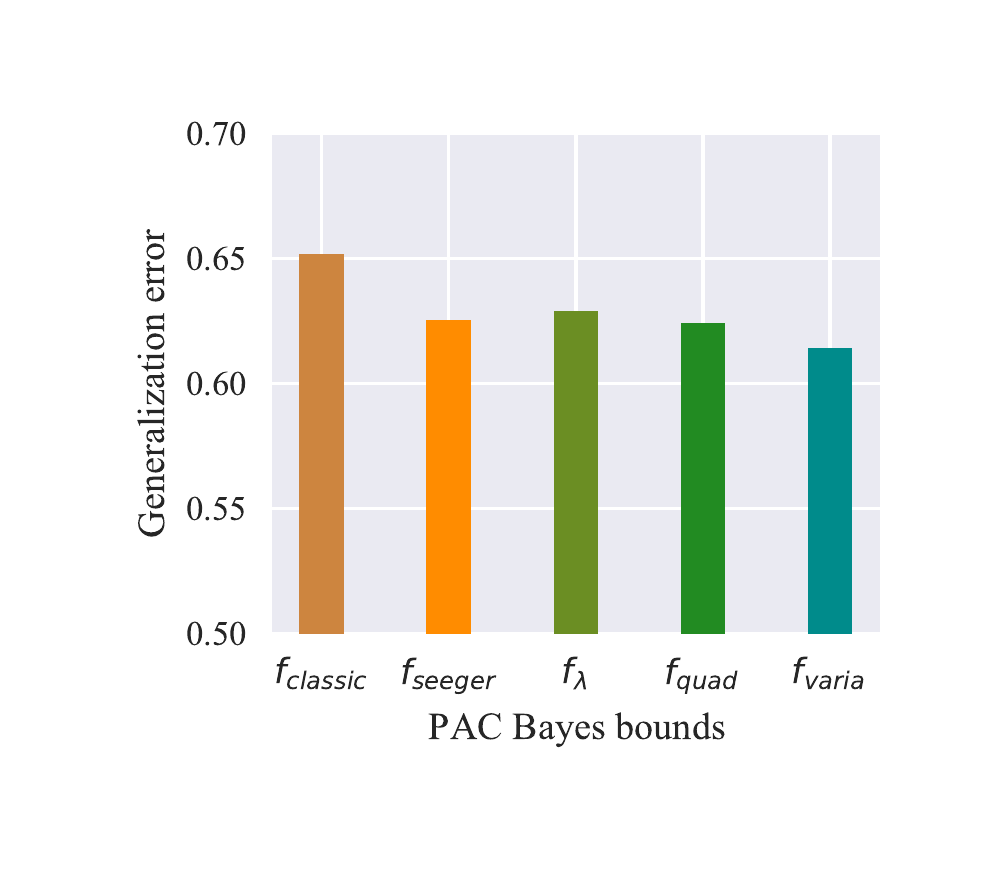}}
    \subfigure[]{
        \label{DP_Labels_bound_testing}
        \includegraphics[width=0.2\textwidth]{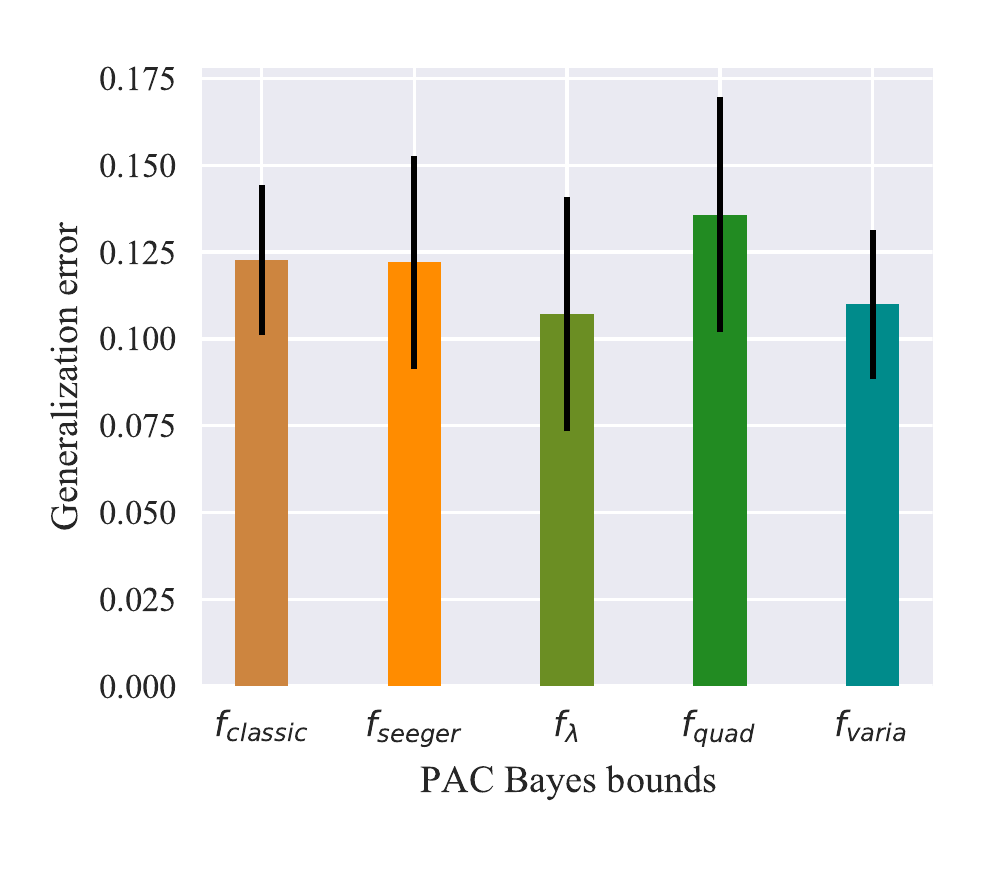}}
	\caption{The average performance of learning new tasks with different number of training tasks. (a) Comparison of different PAC-Bayes bounds in training phase in permuted pixels environment; (b) The average test bound of learning new tasks for permuted pixels with $95\%$ confidence interval; (c) Comparison of different PAC-Bayes bounds in training phase in permuted labels environment; (d) The average test bound of learning new tasks for permuted labels with $95\%$ confidence interval; (e)Comparison of generalization error bounds for permuted pixels with data-dependent prior; (f) The average test bound with data-dependent prior for permuted pixels with $95\%$ confidence interval; (g) Comparison of generalization error bounds for permuted labels with data-dependent prior; (h) The average test bound with data-dependent prior for permuted labels with $95\%$ confidence interval.}
	\label{Performance_analysis}
\end{figure*}

We also compare five meta-learning PAC-Bayes bounds on two different MNIST environments. Those consist of meta-learning McAllester PAC-Bayes bound ($f_{\rm classic}$), meta-learning Seeger PAC-Bayes bound ($f_{\rm Seeger}$), meta-learning PAC-Bayes $\lambda$ bound ($f_{\lambda}$), meta-learning PAC-Bayes quadratic bound ($f_{\rm quad}$) and meta-learning PAC-Bayes variational bound ($f_{\rm varia}$). As shown in Table \ref{Table_analysis_train}, the performance of various training objectives in the training phase with a random prior model is analyzed, in terms of bound, task complexity, meta complexity, empirical loss and estimated error. Figure \ref{Pixels_bound_training} demonstrates that the proposed three meta-learning PAC-Bayes achieve a competitive bound, especially for the PAC-Bayes variational bound.

Furthermore, the performance of a learned meta-learner on new tasks with five training objectives is also established. Table \ref{Table_analysis_test} shows the specific result of generalization performance and accuracy in the testing phase. As indicated in Figure \ref{Pixels_bound_testing}, the proposed meta-learning PAC-Bayes $\lambda$ bound and meta-learning PAC-Bayes variational bound perform tighter generalization error bound. In addition, these two training objectives lead to improved accuracy. 

% Same conclusions also can be drawn in the permuted labels experiment.

\subsection{PAC-Bayes bounds with data-dependent prior}

In this section, meta-learning PAC-Bayes bounds with data-dependent prior algorithms are verified on the permuted MNIST dataset. We experiment both with priors centered at randomly set weights and priors learnt by ERM on a part of dataset. Specifically, training data is randomly divided into two separate datasets: $30\%$ is used to learn a prior model by ERM approach and the remaining data is applied to train the meta-learner. We run about 10 epochs in the training phase and 30 epochs in the testing phase to build the prior.

First, the discrepancy between the prior model with randomly initialized weights and the learned data-dependent prior model is compared in Figures \ref{DP_train_model_compare} and Figure \ref{DP_Labels_model_compare}. It is obviously that, compared with the random prior model, the nature of the parameter learned in the data-dependent prior model is closer to the finally posterior model, which means it can achieve enhanced convergence performance. Besides, as shown in Figure \ref{DP_train_analysis}, where the convergence performance between a random prior model and a data-dependent prior model during the training phase is analyzed. Obviously, the meta-learning PAC-Bayes bound with data-dependent prior demonstrates a faster convergence ability with a series of epochs, in terms of generalization error bound, accuracy, empiric loss, task complexity and meta-complexity.

Furthermore, five meta-learning training objectives on two different MNIST environments are substantiated  (See Table \ref{Table_analysis_train} and Table \ref{Table_analysis_test}). As shown in Figure \ref{DP_Pixels_bound_training} and Figure \ref{DP_Pixels_bound_testing}, comparison between two classical meta-learning PAC-Bayes bounds, the proposed meta-learning PAC-Bayes $\lambda$ bound and meta-learning PAC-Bayes variational bound achieve a competitive generalization performance. The same conclusions can also can be drawn in the testing phase as shown in Figure \ref{DP_Labels_bound_training} and Figure \ref{DP_Labels_bound_testing}.

% \begin{figure}[!htbp]
% \centering
% \includegraphics[width=0.30\textwidth]{DP_train_model_compare.pdf}
% \caption{Prior model parameter comparison between random prior model and learned data-dependent prior model.}
% \label{DP_train_model_compare}
% \end{figure}

% \begin{figure}[!htbp]
% \centering
% \includegraphics[width=0.4\textwidth]{DP_Labels_model_compare.pdf}
% \caption{Comparison of models for permuted labels with data-dependent prior.}
% \label{DP_Labels_model_compare}
% \end{figure}

\begin{figure}[!htbp]
	\centering
	\subfigure[PAC-Bayes bound]{
        \label{DP_train_model_compare}
        \includegraphics[width=0.22\textwidth]{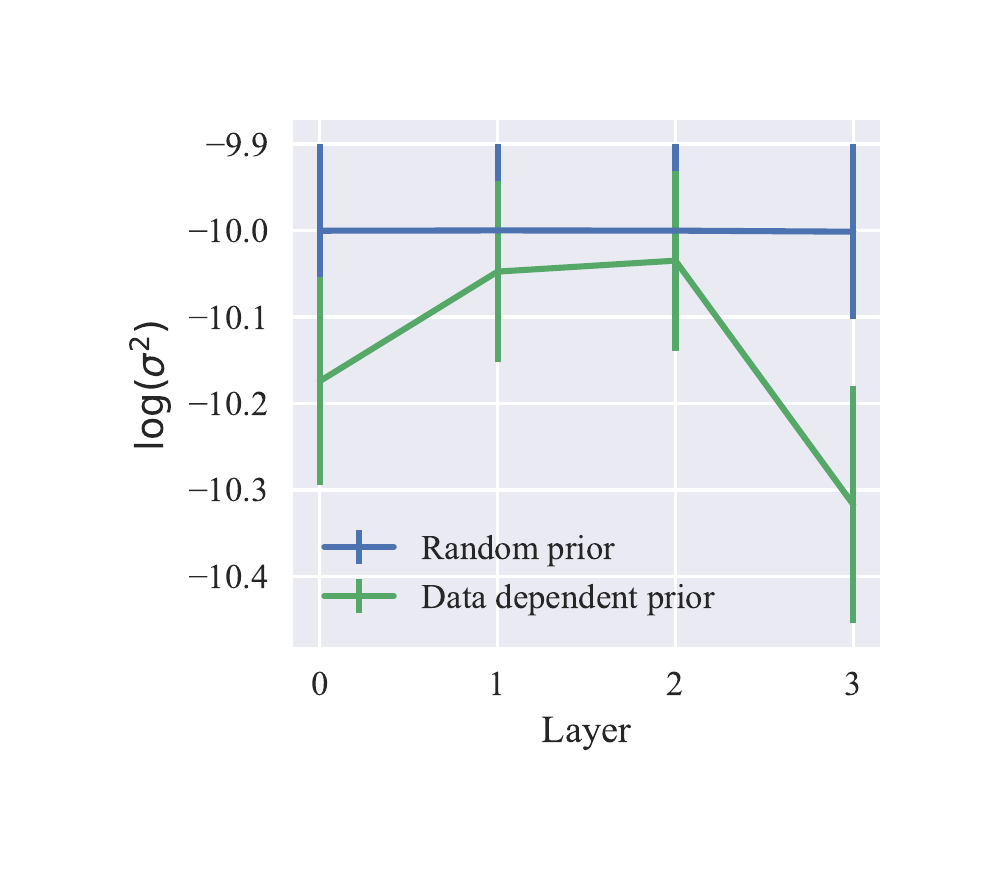}}
	\subfigure[Accuracy]{
	    \label{DP_Labels_model_compare}
        \includegraphics[width=0.24\textwidth]{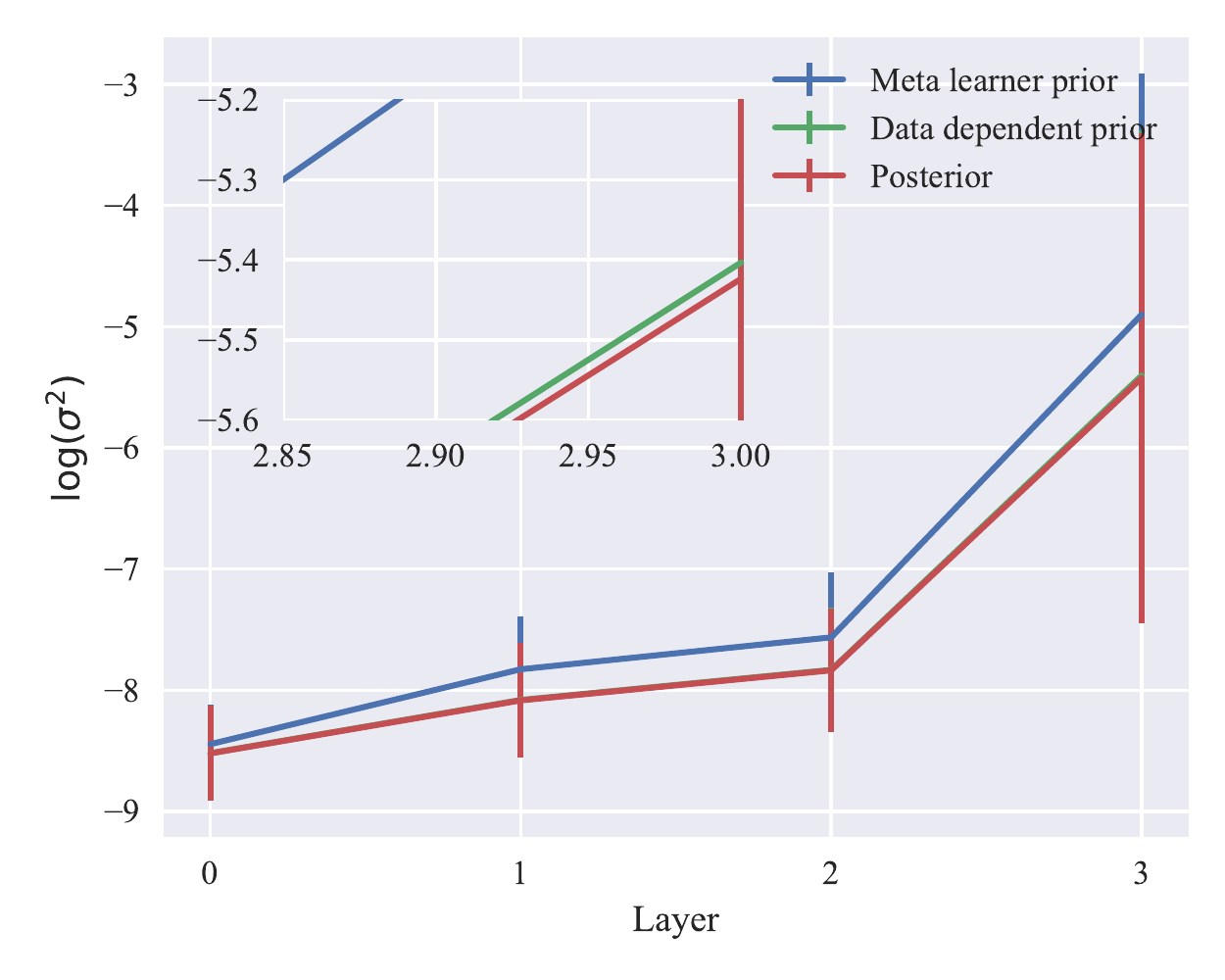}}
	\caption{Parameter analysis of models with different prior. (a)Prior model parameter comparison between random prior model and learned data-dependent prior model; (b) Comparison of models for permuted labels with data-dependent prior. }
	\label{model_compare_parameter}
\end{figure}

\begin{figure*}[!htbp]
	\centering
% 	\vspace{-0.35cm}
% 	\subfigtopskip=2pt 
% 	\subfigbottomskip=2pt
% 	\subfigcapskip=-5pt
	\subfigure[PAC-Bayes bound]{
        \label{DP_train_ana_bound}
        \includegraphics[width=0.22\textwidth]{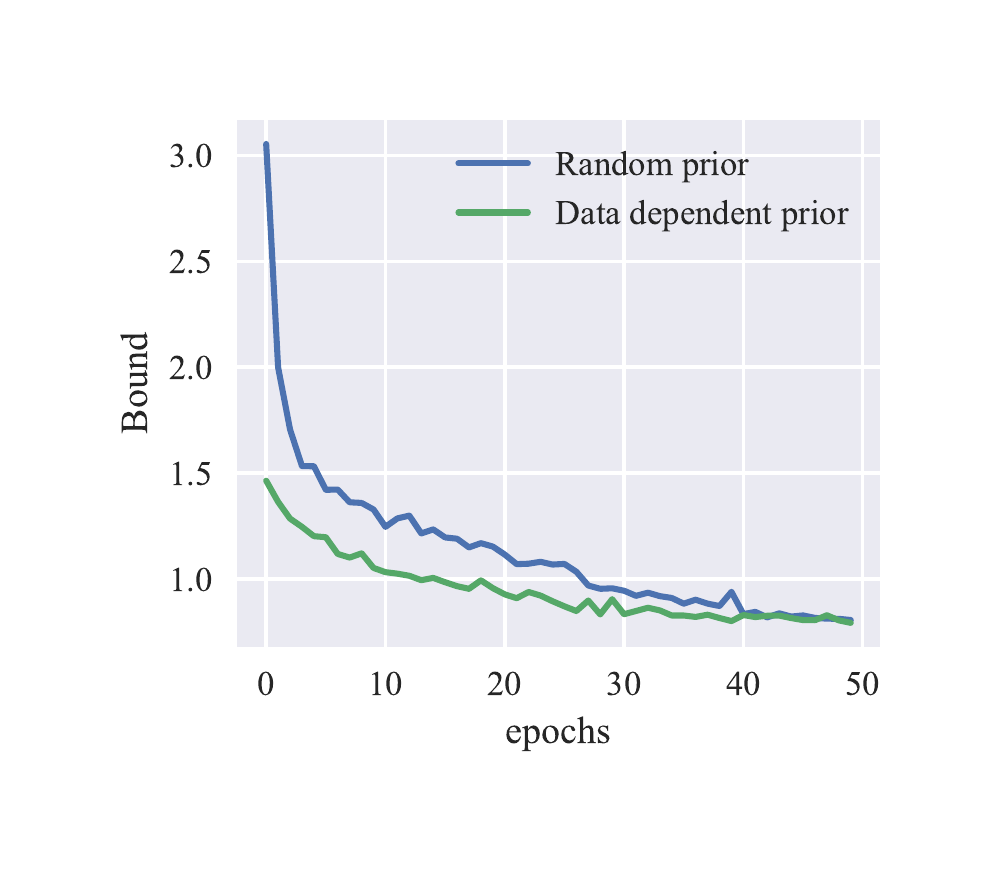}
        }
	\subfigure[Accuracy]{
	    \label{DP_train_ana_accuracy}
        \includegraphics[width=0.22\textwidth]{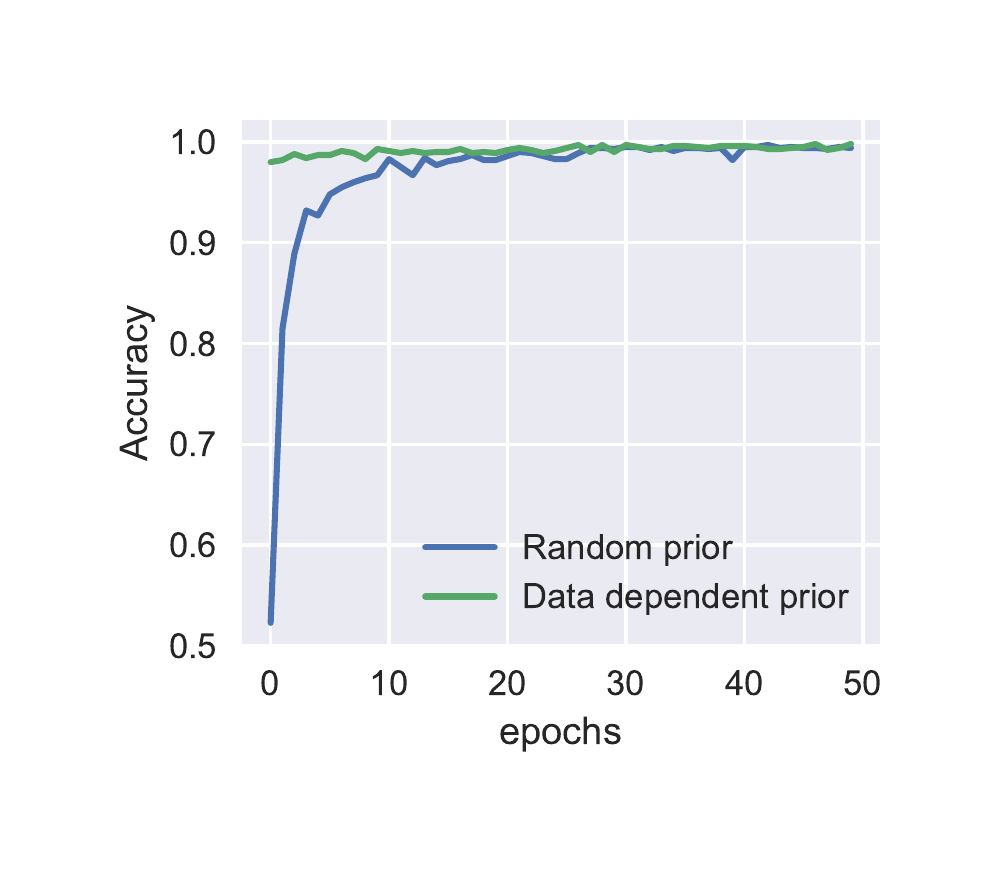}
        }
	\subfigure[Empiric loss]{
        \label{DP_train_ana_empiric_loss}
        \includegraphics[width=0.22\textwidth]{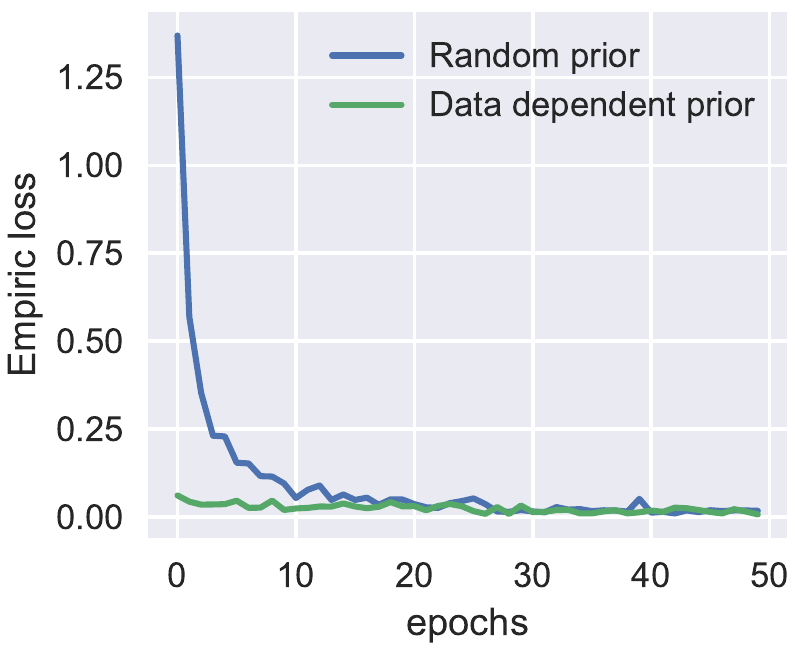}} \\
     \subfigure[Task complexity]{
	    \label{DP_train_ana_Task_Comp}
        \includegraphics[width=0.22\textwidth]{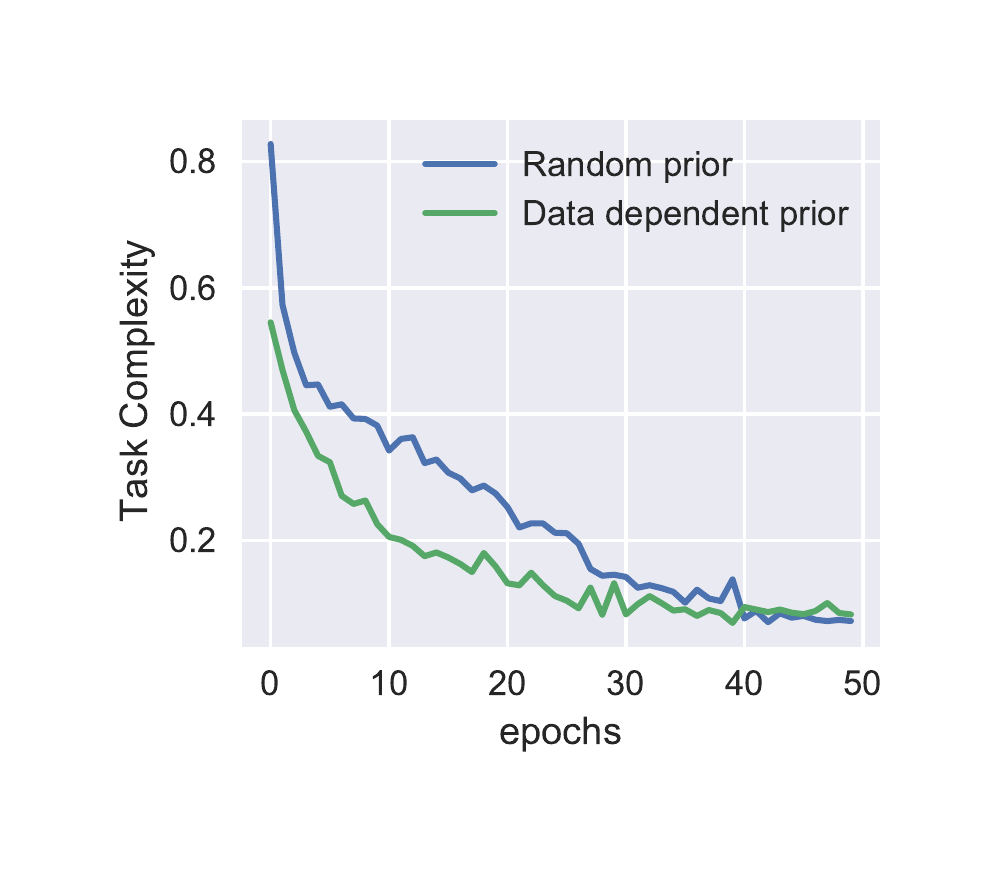}
        }
    %  多个子图跨页显示
    % \end{figure}
    % \addtocounter{figure}{-1} %先欺骗LaTeX图形计数器
    % \begin{figure}
    % \addtocounter{figure}{1} %再告诉LaTeX图形计数器真相
    %  多个子图跨页显示
    % \centering
	\subfigure[Meta complexity]{
        \label{DP_train_ana_Meta_Comp}
        \includegraphics[width=0.22\textwidth]{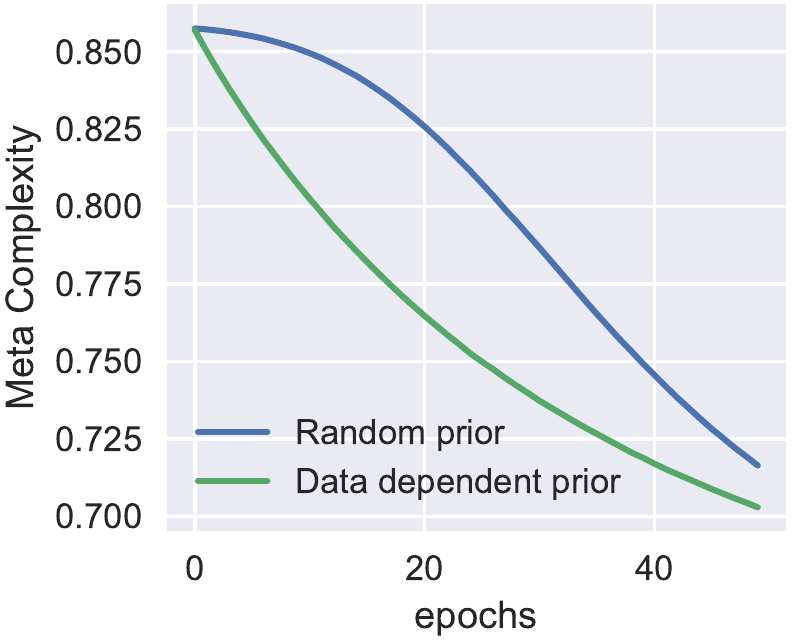}
        }  
	\caption{Performance analysis of meta-learning with data-dependent prior (training objective is PAC-Bayes quadratic bound).}
	\label{DP_train_analysis}
\end{figure*}

% \subsection{Experiment results in permuted labels environment}
% In this section, the experiment results in permuted labels environment are analyzed. 

% First, as shown in Table \ref{Table_analysis_train} and Table \ref{Table_analysis_test}, we analyze the weights uncertainty of prior model by layers in the permuted labels environment.

% We compare five meta-learning training objectives on permuted labels environment in both training phase and testing phase, including $f_{\rm classic}$, $f_{\rm Seeger}$, $f_{\lambda}$, $f_{\rm quad}$ and $f_{\rm varia}$. The experiments are done under settings as described above.

% Furthermore, meta-learning PAC-Bayes bounds with data-dependent prior algorithms are verified on the permuted labels environment. First, the difference between randomly prior model and learned data-dependent prior model is compared in Figure \ref{DP_Labels_model_compare}. Besides, the generalization performance of five training objective are analyzed i both training phase and testing phase.

\section{Conclusions and future work}\label{section_8_conclusion}
In this paper, the meta-learning PAC-Bayes bounds with data-dependent prior algorithms are explored. The proposed theory can be applied to develop a practical algorithm, which can achieve a balance between model accuracy and generalization performance. First, based on the PAC-Bayes relative entropy bound, the meta-learning PAC-Bayes $\lambda$ bound and meta-learning PAC-Bayes quadratic bound are derived. Furthermore, with an eye to achieving improved generalization performance, the meta-learning PAC-Bayes variational bound is also investigated. Those bounds have been applied to develop a practical meta-learning model with generalization performance guarantee and reduced overfitting. Next, in order to improve the convergence ability, combining the ERM approach, meta-learning PAC-Bayes bounds with data-dependent prior algorithms are also proposed.

The results of our experiments on two different MNIST environments, including shuffled pixels and permuted labels, demonstrate that meta-learning PAC-Bayes $\lambda$ bound and meta-learning PAC-Bayes variational bound can achieve competitive performances in terms of generalization error upper bound and estimation accuracy in both training and testing phases. Moreover, meta-learning PAC-Bayes bound with data-dependent prior can attain rapid convergence ability.

In future work, one could further investigate different prior distribution, such as distribution-dependent prior, to achieve faster convergence ability and greater accuracy. We note that the KL term dominates the generalization upper bound of PAC-Bayes theory, and  will explore efficient ways to optimize this term.

% \section*{Acknowledgment}
% This work was supported by the Australian Research Council through the Discovery Project under Grant DP200100700.

%Dr. Reveryrand would like to acknowledge the funding by XLIM, Limoges, France. 
% The authors would like to thank Dr. David Root and Dr. Jean-Pierre Teyssier at Agilent Technologies for the loan of the time-domain nonlinear measurement equipment and TriQuint Semiconductor for the donation of the transistors. 

% or


% ============================================
%\appendices
%Appendix one text goes here %\cite{Roberg2010}.

%\section{}

%\section*{Acknowledgment}

%The authors would like to thank D. Root for the loan of the SWAP. The SWAP that can ONLY be usefull in Boulder...

\ifCLASSOPTIONcaptionsoff
  \newpage
\fi



%\begin{thebibliography}{1}


\bibliographystyle{IEEEtran}
\bibliography{MTT_reveyrand}

\end{document}